\documentclass[conference]{IEEEtran}
\pdfoutput=1
\usepackage{amsmath,amssymb, graphicx}
\usepackage[ruled,linesnumbered]{algorithm2e}
\usepackage{algorithmic}
\usepackage{mdframed}
\usepackage{paralist} 
\usepackage{color}
\usepackage{multirow}
\usepackage{rotating}
\usepackage{amsmath}
\usepackage[mathscr]{eucal} 
\usepackage{amsbsy} 
\usepackage{subfigure}
\usepackage{booktabs}
\usepackage{xcolor}
\usepackage{tikz}
\usetikzlibrary{snakes}
\usepackage{multirow}
\usepackage{epstopdf}
\usepackage{balance}
\usepackage{tablefootnote}
\usepackage{empheq} 
\usepackage{moresize}
\usepackage{array}
\usepackage{enumitem}
\setlist{nolistsep}
\usepackage{amsthm}
\newcommand{\tensor}[1]{\underline{ \mathbf{#1} }}
\usepackage[font=small,labelfont=bf,skip=0pt]{caption}
\DeclareMathOperator*{\argmin}{argmin}

\DeclareCaptionType{copyrightbox}
\usepackage{url}

\newcounter{ALC@tempcntr}

\usepackage{hyperref}

\newcommand{\hide}[1]{}

\newcommand{\ben}{\begin{enumerate*}}
\newcommand{\een}{\end{enumerate*}}
\newcommand{\bit}{\begin{itemize*}}
\newcommand{\eit}{\end{itemize*}}

\newtheorem{lemma}{Lemma}

\def\x{{\mathbf x}}

\newcommand{\effcor}{\textsc{CorConDia}\xspace}
\newcommand{\getrank}{\textsc{GetRank}\xspace}
\newcommand{\method}{\textsc{SamBaTen}\xspace}

\newcommand{\codeurl}{\url{www.cs.ucr.edu/~epapalex/src/SamBaTen.zip}}

\begin{document}

\author{\IEEEauthorblockN{Ekta Gujral}
\IEEEauthorblockA{Computer Science and Engineering\\
University of California Riverside\\
Email: egujr001@ucr.edu}
\and
\IEEEauthorblockN{Ravdeep Pasricha}
\IEEEauthorblockA{Computer Science and Engineering\\
University of California Riverside\\
Email: rpasr001@ucr.edu}
\and
\IEEEauthorblockN{Evangelos E. Papalexakis}
\IEEEauthorblockA{Computer Science and Engineering\\
University of California Riverside\\
Email: epapalex@cs.ucr.edu}}

\title{SamBaTen: Sampling-based Batch Incremental Tensor Decomposition}
\maketitle
\begin{abstract}
Tensor decompositions are invaluable tools in analyzing multimodal datasets. In many real-world scenarios, such datasets are far from being static, to the contrary they tend to grow over time. For instance, in an online social network setting, as we observe new interactions over time, our dataset gets updated in its ``time'' mode. How can we maintain a valid and accurate tensor decomposition of such a dynamically evolving multimodal dataset, without having to re-compute the entire decomposition after every single update? In this paper we introduce \method, a {\em Sam}pling-based {\em Ba}tch Incremental {\em Ten}sor Decomposition algorithm, which incrementally maintains the decomposition given new updates to the tensor dataset. \method is able to scale to datasets that the state-of-the-art in incremental tensor decomposition is unable to operate on, due to its ability to effectively summarize the existing tensor and the incoming updates, and perform all computations in the reduced summary space. We extensively evaluate \method using synthetic and real datasets. Indicatively, \method achieves comparable accuracy to state-of-the-art incremental and non-incremental techniques, while being {\em 25-30 times faster}. Furthermore, \method scales to very large sparse and dense dynamically evolving tensors of dimensions up to $100K \times 100K \times 100K$ where state-of-the-art incremental approaches were not able to operate.
\end{abstract}

\section{Introduction}
\label{sec:intro}
Tensor decomposition is a very powerful tool for many problems in data mining \cite{kolda2005higher}, machine learning \cite{pmlrv51anandkumar16}, chemometrics \cite{bro1997parafac}, signal processing \cite{sidiropoulos2004low} to name a few areas. The success of tensor decomposition lies in its  capability of finding complex patterns in multi-way settings, by leveraging higher-order structure and correlations within the data.  The dominant tensor decompositions are CP/PARAFAC (henceforth referred to as CP), which extracts interpretable latent factors from the data, and Tucker, which estimates the joint subspaces of the tensor. In this work we focus on the CP decomposition, which has been shown to be extremely effective in exploratory data mining time and time again.
\begin{figure}[!ht]
	\begin{center}
		\includegraphics[width = 0.45\textwidth]{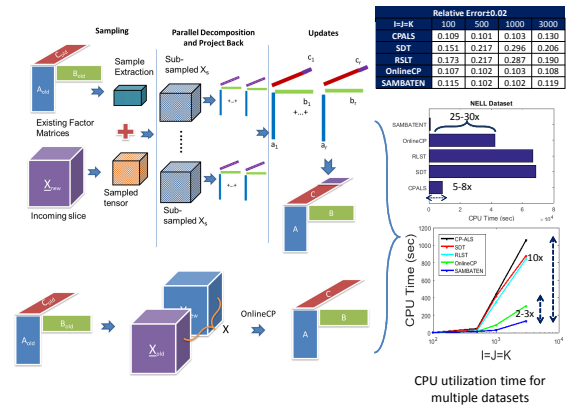}
		\caption{\bf{\method vs state-of-art techniques: Our proposed method \method outperforms state-of-the-art baselines while maintaining competitive accuracy.}}
		\label{fig:crown_jewel}
	\end{center}
	\vspace{-0.1in}
\end{figure}

In a wide variety of modern real-world applications, especially in the age of Big Data, data are far from being static. To the contrary, the data get updated dynamically. In these circumstances, a data tensor needs to be shrunk, expanded or modified on any of its mode. For instance, in an online social network, new interactions occur every second and new friendships are formed at a similar pace. In the tensor realm, we may view a large proportion of these dynamic updates as an introduction of new ``slices'' in the tensor: in the social network example, new interactions that happen as time evolves imply the introduction of new snapshots of the network, which grow the tensor in the ``time'' mode. How can we handle such updates in the data without having to re-compute the decomposition whenever an update arrives, but incrementally update the pre-existing results given the new data?

Computing the decomposition for a dynamically updated tensor is challenging, with the challenges lying, primarily, on two of the three V's in the traditional definition of Big Data: {\em Volume} and {\em Velocity}. As a tensor dataset is updated dynamically, its volume increases to the point that techniques which are not equipped to handle those updates {\em incrementally}, inevitably fail to execute due to the sheer size of the data. Furthermore, even though the applications that tensors have been successful so far  do not require real-time execution per se, the decomposition algorithm must, nevertheless, be able to ingest the updates to the data at a rate that will not result in the computation being ``drowned'' by the incoming updates.

The majority of prior work has focused on the Tucker Decomposition of incrementing tensors \cite{papadimitriou2005streaming,SunITA,fanaee2015multi}, however very limited amount of work has been done on the CP. Nion and Sidiropoulos \cite{nion2009adaptive} proposed  two methods namely Simultaneous Diagonalization Tracking (SDT) and Recursive Least Squares Tracking (RLST) and most recently, \cite{zhou2016accelerating} introduced the OnlineCP decomposition for higher order online tensors. Even though prior work in incremental CP decomposition, by virtue of allowing for incremental updates to the already computed model, is able to deal with {\em Velocity}, when compared to the naive approach of re-executing the entire decomposition on the updated data, every time a new update arrives, it falls short when the {\em Volume} of the data grows.

\begin{mdframed}[linecolor=red!60!black,backgroundcolor=gray!20,linewidth=2pt,topline=false,  rightline=false, leftline=false] 
In this paper we propose a novel large scale incremental CP tensor decomposition that effectively leverages (potential) sparsity of the data, and achieves faster and more scalable performance than state-of-the-art baselines, while maintaining comparable accuracy. 
\end{mdframed}
We show a snapshot of our results in Figure \ref{fig:crown_jewel}: \method is faster than all state-of-the-art methods on data that the baselines were able to operate on. Furthermore, \method was able to scale to, both dense and sparse, dynamically updated tensors, where none of the baselines was able to run. Finally, \method achieves comparable accuracy to  existing incremental and non-incremental methods.
Our contributions are summarized as follows:
\begin{itemize}[noitemsep]
	\item {\bf Novel scalable algorithm}: We introduce  \method, a novel scalable  and  parallel incremental tensor decomposition algorithm for efficiently computing the CP decomposition of incremental tensors. The advantage of the proposed algorithm stems from the fact that it only operates on small summaries of the data at all times, thereby being able to maintain its efficiency regardless of the size of the full data. Furthermore, if the tensor and the updates are sparse, \method leverages the sparsity by summarizing in a way that retains only the useful variation in the data. To the best of our knowledge, this is the first incremental tensor decomposition which effectively leverages sparsity in the data.
	\item {\bf Quality control}: As a tensor is dynamically updated, some of the incoming updates may contain rank-deficient structure, which, if not handled appropriately, can pollute the results. We equip \method with a quality control option, which effectively determines whether an update is rank-deficient, and subsequently handles the update in a way that it does not affect latent factors that are not present in that update.
	\item {\bf Extensive experimental evaluation}: Through experimental evaluation on six real-world datasets with sizes that range up to $70$GB, and synthetic tensors that range up to $100K\times 100K \times 100K$, we show that our method can incrementally maintain very accurate decompositions, faster and in a more scalable fashion than state-of-the-art methods.
\end{itemize}

{\bf Reproducibility}: We make our Matlab implementation publicly vailable at \codeurl. Furthermore, all the datasets we use for evaluation are publicly available.

\section{Problem Formulation}
\label{sec:problem}

\subsection{Preliminary Definitions}

\textbf{Tensor} : A tensor is a higher order generalization of a matrix. In order to avoid overloading the term ``dimension'', we call an $I\times J \times K$ tensor a three ``mode'' tensor, where ``modes'' are the numbers of indices used to index the tensor. The number of modes is also called ``order''. Table \ref{table:t2} contains the symbols used throughout the paper. We refer the interested reader to several surveys that provide more details and a wide variety of tensor applications \cite{kolda2009tensor,papalexakis2016tensors}. In the interest of space, we also refer the reader to \cite{papalexakis2016tensors} for the definitions of Kronecker and Khatri-Rao products which are not essential for following the basic derivation of our approach.

\textbf{Slice} :  A slice is a (m-1)-dimension partition of tensor where an index is varied in one mode and the indices fixed in the other modes. There are three categories of slices : horizontal ($\tensor{X}$(i,:,:)) , lateral ($\tensor{X}$(:,j,:)), and frontal ($\tensor{X}$(:,:,k)) for third-order tensor X as shown in Figure \ref{fig:slice}. 

\begin{figure}[!ht]
	\begin{center}
		\includegraphics[width  = 0.3\textwidth]{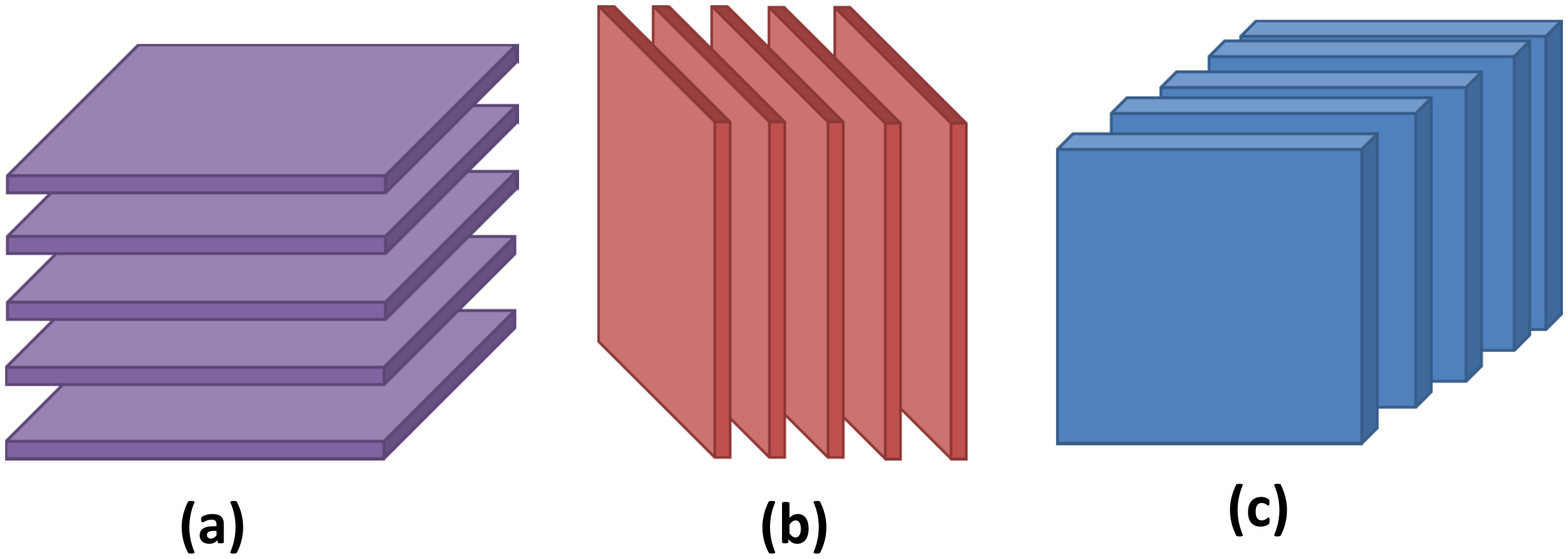}
				\caption{ \bf{Slices of 3-order tensor (a) horizontal $\tensor{X}(i,:,:)$ (b)  lateral $\tensor{X}(:,j,:)$, (c) frontal $\tensor{X}(:,j,:)$.}}
		\label{fig:slice}
	\end{center}
	\vspace{-0.2in}
\end{figure}

\begin{table}[h!]
\begin{center}
\begin{tabular}{ |c|c| }
\hline
Symbols & Definition \\ 
\hline
\hline
$\tensor{X},\mathbf{X}, \mathbf{x},x$ & Tensor, Matrix, Column vector, Scalar \\ 
\hline
$\mathbb{R}$ & Set of Real Numbers  \\ 
\hline
$\circ$ & Outer product  \\ 
\hline
$\lVert \mathbf{A} \rVert_F, \| \mathbf{a} \|_2$& Frobenius norm, $\ell_2$ norm \\
\hline 
$\mathbf{x}(I)$ & Spanning the elements of $\mathbf{x}$ in indices $\in I$ \\ 
\hline
$\mathbf{x}(:)$ & Spanning all elements of $\mathbf{x}$\\ 
\hline
$\mathbf{X}(:,r)$ &$ r^{th}$ column of $\mathbf{X}$  \\ 
\hline
$\mathbf{X}(r,:)$ & $ r^{th}$ row of $\mathbf{X}$  \\ 
\hline
 $\otimes $& Kronecker product\\
\hline
 $\odot$ & Khatri-Rao product (column-wise Kronecker product \cite{papalexakis2016tensors})\\
\hline
\end{tabular}
\bigskip
\caption{\bf{Table of symbols and their description}}
\label{table:t2}
\end{center}
\vspace{-0.2in}
\end{table}
\textbf{Canonical Polyadic Decomposition} : One of the most popular and widely used tensor decompositions is the Canonical Polyadic (CP) or CANDECOMP/PARAFAC decomposition \cite{carroll1970analysis,PARAFAC}. We henceforth refer to this decomposition as CP. In CP., the tensor is decomposed into a sum of rank-one tensors, i.e., a sum of outer products of three vectors (for three-mode tensors):
$
\tensor{X} \thickapprox \sum_{r=1}^R \mathbf{A}(:,r) \circ \mathbf{B}(:,r)\circ \mathbf{C}(:,r)
$
where $\mathbf{A}\in  \mathbb{R}^{I\times R}, \mathbf{B} \in \mathbb{R}^{J \times R}, \mathbf{C} \in \mathbb{R}^{K\times R}$, and the outer product is given by $(\mathbf{A}(:,r) \circ \mathbf{B}(:,r)\circ \mathbf{C}(:,r))(i,j,k) = \mathbf{A}(i,r) \mathbf{B}(j,r) \mathbf{C}(k,r)$ for all $i,j,k$. In order to compute the decomposition we typically need to minimize the squared differences (i.e., Frobenius norm) between the original tensor and the model
There exist other modeling approaches in the literature \cite{chi2012tensors} which minimize the KL-Divergence, however, Frobenius norm-based approaches are still to this day the most well studied. We reserve investigation of other loss functions as future work.

\subsection{Problem Definition}
In many real-world applications, data grow dynamically and may do so in many modes. For example, given a dynamic tensor in a location-based recommendation system, as shown in Figure \ref{fig:realLife}(a), structured as location $\times$ hot-spots $\times$ people, the number of registered location, hot-spots , and people visited may all increase over time. Another example is time-evolving social network interactions figure \ref{fig:realLife} (b) as also described in Introduction section. This incremental property of data  gives rise to the need for an on-the-fly update of the existing decomposition, which we name incremental tensor decomposition. Notice that the literature (and thereby this paper) uses the terms ``incremental'', ``dynamic'', and ``online'' interchangeably.  In such scenarios, data updates happen very fast which make traditional (non-incremental) methods to collapse because they need to recompute the entire large scaled data.
\begin{figure}[!ht]
	\begin{center}
		\includegraphics[  width  = 0.36\textwidth]{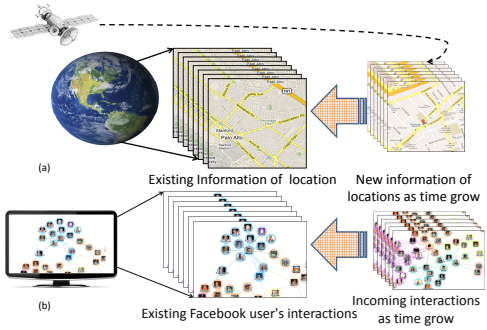}
				\caption{ \bf{Real life dynamic data examples (a) collecting new location update from satellite to GPS recommendations systems (b) growing Facebook interaction between people over the time.}}
		\label{fig:realLife}
	\end{center}
	\vspace{-0.1in}
\end{figure}

This paper address the problem of large scale incremental tensor decomposition. Without loss of generality , we focus on a 3-mode tensor one of whose modes is growing with time. However, the problem definition (and our proposed method) extends to any number of modes. Let us consider $\tensor{X}(t)$=$\mathbb{R}^{I\times J \times K_1(t)}$ at time $t$. The CP decomposition of $\tensor{X}(t)$ is given as :
\[
\begin{array}{c}
\tensor{X}^{(1)}(t)\approx (\mathbf{A}(t) \odot \mathbf{B}(t))\mathbf{C}^{T}(t) \approx \mathbf{L}(t)\mathbf{C}^{T}(t)
\end{array}
\]
where $\mathbf{L}(t)=(\mathbf{A}(t) \odot \mathbf{B}(t))$ of dimension $IJ\times R$ and $\mathbf{C}^{T}(t)$ is of dimension $K_{1}\times R$. When new incoming slice $\tensor{X}(t^{'})$=$\mathbb{R}^{I\times J \times K_2(t^{'})}$ is added in mode 3, required decomposition at time $t^{'}$ is  :
\[
\begin{array}{c}
\tensor{X}^{(1)}(t+t^{'}) \approx \mathbf{L}(t+t^{'})\mathbf{C}^{T}(t+t^{'})
\end{array}
\]
where $\mathbf{L}(t+t^{'})=(\mathbf{A}(t+t^{'}) \odot \mathbf{B}(t+t^{'}))$ of dimension $IJ\times R$ and $\mathbf{C}^{T}(t+t^{'})$ is of dimension $(K_{1}+K_{2})\times R$.\\

The problem that we solve is the following:

\begin{mdframed}[linecolor=red!60!black,backgroundcolor=gray!20,linewidth=2pt,    topline=false,rightline=false, leftline=false] 
\begin{problem}
{\bf Given} (a) pre-existing set of decomposition results $\mathbf{A}(t),\mathbf{B}(t)$ and $\mathbf{C}(t)$ of $R$ components, which approximate tensor $\tensor{X}_{old}$ of size $ I \times J \times K_1$ at time \textit{t} , (b) new incoming slice in form of tensor $\tensor{X}_{new}$ of size $ I \times J \times K_2$ at any time $t^{'}$, find updates of $\mathbf{A}(t^{'}),\mathbf{B}(t^{'})$ and $\mathbf{C}(t^{'})$ {\bf incrementally} to approximate tensor $\tensor{X}$ of dimension $ I \times J \times K$, where K =$ K_1 + K_2$ after appending new slice or tensor to $3^{rd}$ dimension while maintaining a comparable accuracy with running the full CP decompositon on the entire updated tensor $\tensor{X}$.
\end{problem}
\end{mdframed}

To simplify notation, we will interchangeably refer to $\mathbf{A}(t)$ as $\mathbf{A}_{old}$ (when we need to refer to specific indices of that matrix), and similarly for $\mathbf{A}(t')$ we shall refer to it as $\mathbf{A}^{'}$.

\section{Proposed Method: \method}
\label{sec:method}
As we mention in the introduction, there exists a body of work in the literature that is able to efficiently and incrementally update the CP decomposition in the presence of incoming tensor slices \cite{nion2009adaptive,zhou2016accelerating}. However, those methods fall short when the size of of the dynamically growing tensor increases, and eventually are not able to scale to very large dynamic tensors. The reason why this happens is because these methods operate on the {\em full data}, and thus, even though they incrementally update the decomposition (avoiding to re-compute it from scratch), inevitably, as the size of the full data grows, it takes a toll on the run-time and scalability.

In this paper we propose \method, which takes a different view of the solution, where instead of operating on the full data, it operates on a summary of the data. Suppose that the ``complete'' tensor (i.e., the one that we will eventually get when we finish receiving updates) is denoted by $\tensor{X}$. Any given incoming slice (or even a batch of slice updates) can be, thus, seen as a sample of that tensor, $\tensor{X}$ where the sampled indices in the third mode (which we assume is the one receiving the updates) are the indices of the incoming slice(s).  Suppose, further, that given a set of sample tensors (which are drawn by randomly selecting indices from all the modes of the tensor) we can approximate the original tensor with high-accuracy (which, in fact, the literature has shown that it is possible \cite{papalexakis2012parcube,erdos2013walk}). Therefore, when we receive a new set of slices as an update, if we update those samples with the new indices, then we should be able to compute a decomposition very efficiently which incorporates the slice updates, and approximates the updated tensor well. This line of reasoning inspires \method, a visual summary of which is shown in Figure \ref{fig:method}.

\begin{figure}[!ht]
\begin{center}
		\includegraphics[clip, trim=1.2cm 8.7cm 1.5cm 0.5cm,width  = 0.45\textwidth]{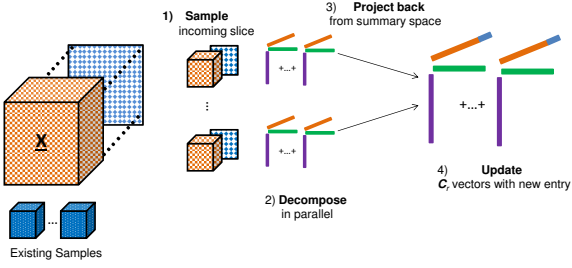}
		\caption{\bf{\method: Sampling-based Batch Incremental Tensor Decomposition: 1) Sample incoming tensor into sub-tensors,  2) run parallel decompositions on the samples, 3) project back the results into the original space, and, finally, 4) update the incrementally growing factor matrix $\mathbf{C}$.}}
		\label{fig:method}
	\end{center}
	\vspace{-0.1in}
\end{figure}

\subsection{The heart of \method}
The algorithmic framework we propose is shown in Figure \ref{fig:method} and is described below:
We assume that we have a pre-existing set of decomposition results before the update, as well as a set of {\em summaries} of the tensor before the update. Summaries are in the form of sampled sub-tensors, as described in the text below. For simplicity of description, we assume that we are receiving updated slices on the third mode, which in turn have to add new rows to the $\mathbf{C}$ matrix (that correspond to the third mode). We, further, assume that the updates come in batches of new slices, which, in turn, ensures that we see a mature-enough update to the tensor, which contains useful structure. Trivially, however, \method can operate on singleton batches.

In the following lines, $\tensor{X}$ is the tensor prior to the update and $\tensor{X}_{new}$ is the batch of incoming slices. Given an incoming batch, \method performs the following steps:

\textbf{Sample}: The rationale behind \method is that each batch $\tensor{X}_{new}$ can be seen as a sample of third-mode indices of (what is going to be) the full tensor. In this step, we are going to merge these incoming indices with an already existing set of sampled tensors. In order to obtain those pre-existing samples, we follow a similar approach to \cite{papalexakis2012parcube}. Namely, we sample indices from the tensor $\tensor{X}$ based on a measure of importance.
To determine the importance for each mode $m$ and then sample the indices using this measure as a sampling weight divided by its probability. An appropriate measure of importance (MoI) is the \textbf{sum-of-squares} of the tensor for each mode.  For the first mode , MoI is defined as:
\begin{equation}\label{eq:1}
x_a(i)=\sum_{j=1}^J \sum_{k=1}^K \tensor{X}(i,j,k)^2  
\end{equation} 
for  $i \in$ (1,I). Similarly, we can define the MoI for modes 2 and 3.
\hide{
\begin{equation}\label{eq:2}
x_b(j)=\sum_{i=1}^I \sum_{k=1}^K \tensor{X}(i,j,k)^2, x_c(k)=\sum_{i=1}^I \sum_{j=1}^J \tensor{X}(i,j,k)^2  
\end{equation}
for  $j \in$ (1, J) and $k \in$  (1,K).
}

We sample each mode of $\tensor{X}$ without replacement, using Eq. \ref{eq:1} to bias the sampling probabilities. With $s$ as sampling factor, i.e. if $\tensor{X}$ has size $I \times J \times K$, then $\tensor{X}_s$ will be of size $\frac{I}{s},\frac{J}{s},\frac{K}{s}$. Sampling rate for each mode is independent from each other, and in fact, different rates can be used for imbalanced modes.

In the case of sparse tensors, the sample will focus on the dense regions of the tensor which contains most of the useful structure. In the case of dense tensors, the sample will give priority to the regions of the tensor with the highest variation.

After forming the sample summary $\tensor{X}_s$ for $\tensor{X}$, we merge it with the samples obtained from the intersection of the third-mode indices of $\tensor{X}_{new}$ and the already sampled indices in the remaining modes, so that the final sample is equal to
$\tensor{X}_s$=$\tensor{X}(I_s,J_s,K_s \cup [K+1 \cdots K_{new}])$, where $K+1 \cdots K_{new}$ are the third-mode indices of $\tensor{X}_{new}$.

Due to the randomized nature of this summarization, we need to draw multiple samples, in order to obtain a reliable set of summaries. Each such independent sample is denoted as $\tensor{X}_s^{(r)}$. In the case of dense tensors, obtaining multiple, independent random samples helps summarize as much of the useful variation as possible. In fact, we will see in the experimental evaluation that increasing the number of samples, especially for dense tensors, improves accuracy. 

In \cite{papalexakis2012parcube} the authors note that in order for their method to work, a set of anchor indices must be common between all samples, so that, later on, we can establish the correct correspondence of latent factors. However, in \method we do not have to actively fix a set of indices across sampling repetitions. When we sample $I_s,J_s,K_s$ each time, those indices correspond to a portion of the decomposition that is already computed. Therefore, the entire set of indices $I_s,J_s,K_s$ can serve as the set of anchors. This is a major advantage compared to \cite{papalexakis2012parcube}, since \method 1) does not need to commit to a set of fixed indices for all samples a-priori, which, due to randomization may happen to represent a badly structured portion of the tensor, 2) does not need to be restricted in a ``small'' set of fixed common indices (which is required in \cite{papalexakis2012parcube} in order to ensure that sufficiently enough new indices are sampled across repetitions), but to the contrary, is able to use a larger number of anchor indices to establish correspondence more reliably, and 3) does not require any synchronization between different sampling repetitions, which results in higher parallelism potential.

\begin{center}
\begin{algorithm} [!htp]
\small
\caption{\method} \label{alg:proposed}
\begin{algorithmic}[1]
\REQUIRE Tensor $\tensor{X}_{new}$ of size $I \times J \times K_{new}$, Factor matrices $\mathbf{A}_{old},\mathbf{B}_{old}, \mathbf{C}_{old}$ of size $I \times R$, $J \times R$ and $K_{old} \times R$ respectively, sampling factor $s$ and number of repetitions $r$.
\ENSURE Factor matrices $\mathbf{A}, \mathbf{B}, \mathbf{C}$ of size $I \times R$, $J \times R$ and $(K_{new}+K_{old}) \times R$, $\lambda$.

\FOR{$i = 1$ \TO $r$} 
\STATE Compute $\mathbf{x}_a, \mathbf{x}_b$ and $\mathbf{x}_c$.
\STATE Sample a set of indices $I_s,J_s,K_s$ from $\tensor{X}$ without replacement using $\mathbf{x}_a(i)/\sum\limits_{i=1}^I x_a(i) $ as probability (accordingly for the rest).
\STATE  $\tensor{X}_s$=$\tensor{X}(I_s,J_s,K_s \cup [K+1 \cdots K_{new}])$
\STATE $[\mathbf{A}_i^{'}, \mathbf{B}_i^{'}, \mathbf{C}_i^{'}] = $ CP$\left( \tensor{X}_s, R \right)$.
\STATE Normalize $\mathbf{A}_i^{'}, \mathbf{B}_i^{'}, \mathbf{C}_i^{'}$ (as shown in the text) and absorb scaling in $\boldsymbol{\lambda}$.
\STATE Compute optimal matching between the columns of $\mathbf{A}_{old}, \mathbf{B}_{old}, \mathbf{C}_{old}$ and $\mathbf{A}_i^{'}, \mathbf{B}_i^{'}, \mathbf{C}_i^{'}$ as discussed in the text
\STATE Update only zero entries of $\mathbf{A}, \mathbf{B}, \mathbf{C}$ that correspond to sampled entries of $\mathbf{A}_i^{'}, \mathbf{B}_i^{'}, \mathbf{C}_i^{'}$

\STATE Obtain $\mathbf{C}_{new}$ of size $K_{new} \times R$ by taking the last $K_{new}$ rows of $\mathbf{C}^{'}$. 
\STATE Use a shared copy of $\mathbf{C}_{new}$ and average out its entries in a column-wise fashion across different sampling repetitions. 
 \ENDFOR
 \STATE Update $\mathbf{C}$ of size $(K_{new}+K_{old}) \times R$ as :
$$C =\begin{bmatrix}
           \mathbf{C}_{old} \\
          \mathbf{C}_{new}\\
         \end{bmatrix}$$
       
       \STATE Update scaling $\boldsymbol{\lambda}$ as the average of the previous and new value.
\RETURN A,B,C,$\lambda$
\end{algorithmic}
\end{algorithm}
\end{center}
\textbf{Decompose}: Having obtained $\tensor{X}_s$, from the previous step, \method decomposes the summary using any state-of-the-art algorithm, obtaining factor matrices $[\mathbf{A}_i^{'}, \mathbf{B}_i^{'}, \mathbf{C}_i^{'}]$. For the purposes of this paper, we use the Alternating Least Squares (ALS) algorithm for the CP decomposition, which is probably the most well studied algorithm for CP. The rank used here is equal to the universal rank $R$, however, in Section \ref{sec:getrank} we will discuss whether this choice is always the most appropriate and what alternatives there are.

\textbf{Project back}: The CP decomposition is unique (under mild conditions) up to permutation and scaling of the components \cite{ten2002uniqueness}. This means that, even though the existing decomposition $[\mathbf{A}_{old},\mathbf{B}_{old}, \mathbf{C}_{old}]$ may have established an order of the components, the decomposition $[\mathbf{A}_i^{'}, \mathbf{B}_i^{'}, \mathbf{C}_i^{'}]$ we obtained in the previous step is very likely to introduce a different ordering and scaling, as a result of the aforementioned permutation and scaling ambiguity. Formally, the sampled portion of the existing factors and the currently computed factors are connected through the following relation:
\begin{align*}
[\mathbf{A}_{old}(I_s,:),\mathbf{B}_{old}(J_s,:), \mathbf{C}_{old}(K_s,:)] =\\
 [\mathbf{A}_i^{'}(I_s,:) \boldsymbol{\Lambda \Pi}, \mathbf{B}_i^{'}(J_s,:) \boldsymbol{\Pi}, \mathbf{C}(K_s,:)_i^{'} \boldsymbol{\Pi}]
\end{align*}
where $\boldsymbol{\Lambda}$ is a diagonal scaling matrix (which without loss of generality we absorb on the first factor), and $\boldsymbol{\Pi}$ is a permutation matrix that permutes the order of the components (columns of the factors).

In order to tackle the scaling ambiguity, we need to normalize the results in a consistent manner. In particular, we normalize such that each column of the newly computed factors which spans the indices that are shared with $[\mathbf{A}_{old},\mathbf{B}_{old}, \mathbf{C}_{old}]$ has unit norm:
$\mathbf{A}_i^{'}(:,f)=\frac{\mathbf{A}_i^{'}}{||\mathbf{A}_i^{'}(I_s,f)||_2},$
and accordingly for the remaining factors. Note that for $\mathbf{A}_i^{'}$,  trivially $\mathbf{A}_i^{'}(I_s,f) = \mathbf{A}_i^{'}(:,f)$ and similarly for $\mathbf{B}_i^{'}$. After normalizing, the relation between the existing factors and the currently computed is $\mathbf{A}_{old}(I_s,:) = \mathbf{A}_i^{'} \boldsymbol{\Pi}$ (and similarly for the remaining factors). Each iteration retains a copy of $[\mathbf{A}_{old}(I_s,:),\mathbf{B}_{old}(J_s,:), \mathbf{C}_{old}(K_s,:)]$ which will serve the anchor for disambiguating the permutation of components. We normalize $[\mathbf{A}_{old}(I_s,:),\mathbf{B}_{old}(J_s,:), \mathbf{C}_{old}(K_s,:)]$ to unit norm as well, and the reason behind that lies in the following Lemma:

\begin{lemma}
Consider $\mathbf{a} = \mathbf{A}_i^{'}(:,f_1)$ and $\mathbf{b} = \mathbf{A}_{old}(:,f_2)$. If $f_1$ and $f_2$ correspond to the same latent CP factor, in the noiseless case, then $\mathbf{a}^T \mathbf{b} = 1$ otherwise  $\mathbf{a}^T \mathbf{b} < 1$.
\end{lemma}
\begin{proof}
From Cauchy-Schwartz inequality  $\mathbf{a}^T \mathbf{b} \leq \|\mathbf{a}\|_2 \|\mathbf{b}\|_2$. The above inequality is maximized when $\mathbf{a} = \mathbf{b}$ and for unit norm $\mathbf{a}, \mathbf{b}$, $\mathbf{a}^T \mathbf{b} \leq 1$. Therefore, if $\mathbf{a} = \mathbf{b}$, which happens when $f_1$ and $f_2$ correspond to the same latent CP factor, $\mathbf{a}^T \mathbf{b} =  1$
\end{proof}

Given the above Lemma, we have a guide for identifying the permutation matrix $\boldsymbol{\Pi}$: For every column of $\mathbf{A}_i^{'}$ we compute the inner product with every column of $\mathbf{A}_{old}(I_s,:)$ and compute a matching when the inner product is equal (or close) to 1. Given a large-enough number of rows for  $\mathbf{A}_i^{'}$ (which is usually the case, since we require a large-enough sample of the tensor in order to augment it with the update and compute the factor updates accurately), this matching can be computed more reliably than past approaches that use a related means of permutation disambiguation \cite{papalexakis2012parcube} but rely on a very small number of shared rows which results in sub-optimal results in the presence of noise.

\textbf{Update results}: 

After appropriately permuting the columns of $\mathbf{A}_i^{'}, \mathbf{B}_i^{'}, \mathbf{C}_i^{'}$, we have all the information needed to update our model. Returning to the problem definition of Section \ref{sec:problem}, $\mathbf{A}_i^{'}$ contains the updates to the rows within $I_s$ for $\mathbf{A}(t)$ (and similarly for $\mathbf{B}$ and $\mathbf{C}$). Even though $\mathbf{A, B}$ do not increase their number of rows over time, the incoming slices may contribute valuable new estimates to the already estimated factors. Thus, for the already existing portions of $\mathbf{A, B, C}$ we only update the zero entries that fall within the range of $I_s, J_s$, and $K_s$ respectively. Finally, $ \mathbf{C}_i^{'}([K+1 \cdots K_{new}],:)$ contains the factors for the newly arrived slices, which need to be merged to the already existing columns of $\mathbf{C}$. Since we have properly permuted the columns of $\mathbf{C}_i^{'}$, we accumulate the lower portion of the  $\mathbf{C}_i^{'}$ (which corresponds to the newly added slices) into $\mathbf{C}_{new}$ and we take the column-wise average of the rows to-be-appended to $\mathbf{C}$, across repetitions. Finally, we update
$
   \mathbf{C}(t') =\begin{bmatrix}
             \mathbf{C}_{old} \\
          \mathbf{C}_{new}\\
         \end{bmatrix}.
         $

\subsection{Dealing with rank deficient updates}
\label{sec:getrank}
So far, the above algorithm description is based on the assumption that each one of the sampled tensors $\tensor{X}_s$ we obtain are full-rank, and the CP decomposition of that tensor is identifiable (assumption that is also central to previous works \cite{papalexakis2012parcube}). However, this assumption glosses over the fact that in reality, updates to our tensor may be rank-deficient. In other words, even though $\mathbf{A}(t),\mathbf{B}(t),\mathbf{C}(t)$ have $R$ components, the update may contain $R_{new}$ components, where $R_{new} < R$. If that happens, then the matching as described above is going to fail, and this inevitably leads to very low quality results. Here, we tackle this issue by adding an extra layerincludegraphics of quality control when we compute the CP decomposition, right before line 5 of Algorithm \ref{alg:proposed}: we estimate the number of components $R_{new}$ in $\tensor{X}_s$ and instead of the ``universal'' rank $R$, which may result in low quality factors, we use $R_{new}$ and we accordingly match only those $R_{new}$ to their most likely matches within the existing $R$ components.  Estimating the number of components is a very hard problem, however, there exist efficient heuristics in the literature, such as the Core Consistency Diagnostic (\effcor) \cite{bro2003new} which gives a quality rating for a computed CP decomposition. By successively trying different candidate ranks for $\tensor{X}_s$, we estimate its actual rank, as shown in Algorithm  \ref{alg:getRank} (\getrank), and use that instead. For efficiency we use a recent implementation of \effcor that is especially tailored for exploiting sparsity \cite{papalexakis2015fast}. In the experimental evaluation we demonstrate that using \getrank indeed results in higher-quality latent factors
\hide{
\begin{figure}[!ht]
	\begin{center}
		\includegraphics[clip, trim=1.2cm 6cm 1.2cm 6cm, width  = 0.46\textwidth]{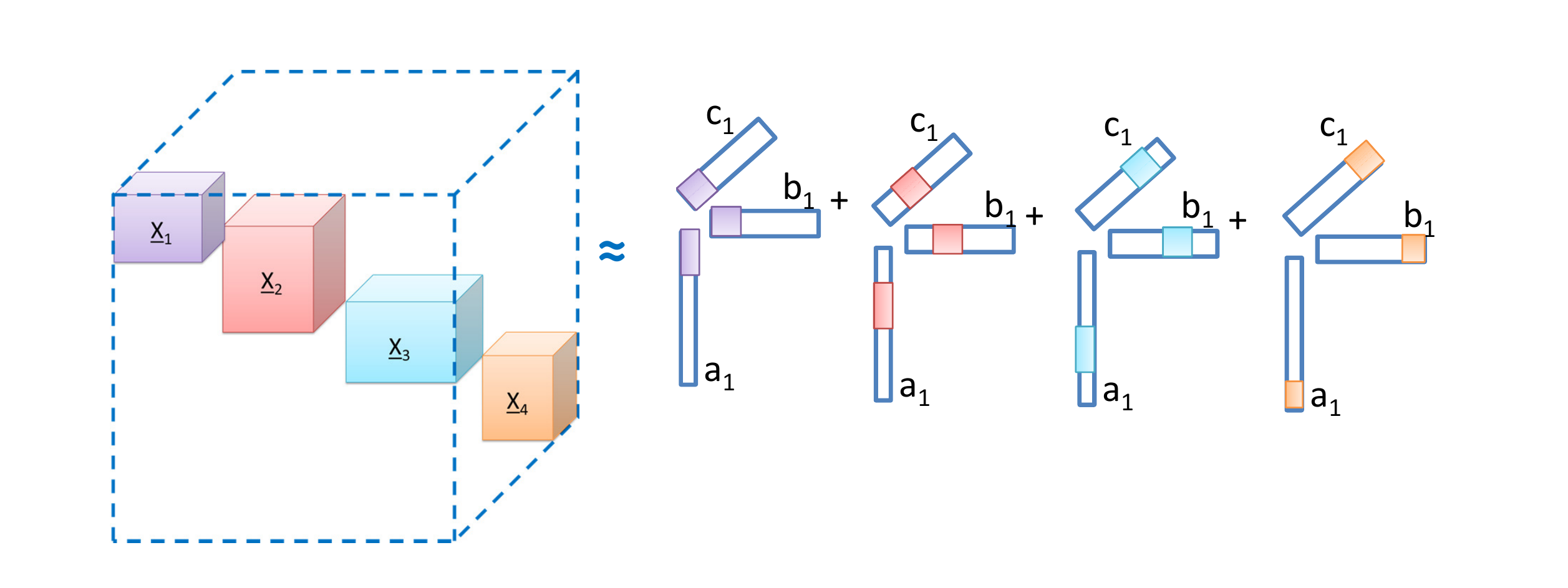}
		\caption{\bf{Tensor with four rank-deficient sub-tensors.}}
		\label{fig:rankDif}
	\end{center}
	\vspace{-0.1in}
\end{figure}
}

\begin{center}
\begin{algorithm} [!htp]
\small
\caption{\getrank} \label{alg:getRank}
\begin{algorithmic}[1]
\REQUIRE Tensor $\tensor{X}$, maximum rank $R$ , maximum no. of iterations $it$.
\ENSURE $R_{new}$
\FOR{1 \TO $R$ }
\FOR{$j=1$ \TO $it$ } 
 \STATE {Run CP Decomposition on $\tensor{X}$ with rank $i$ and obtain $f_i$} 
 \STATE {Run \effcor($\tensor{X}_s$,$f_i$) and obtain $p(i,j)$} 
 \ENDFOR
  \ENDFOR
\STATE sort $p$ and get top 1 index $idx_1$.
\RETURN $R_{new}$ =$idx_1$
\end{algorithmic}
\end{algorithm}
\end{center}

\hide{
\textbf{Summary}: For any third-order tensor that increments with time, we propose an efficient and faster algorithm for incrementally decomposing the sparse as well as dense tensor. We name this algorithm as \method, consist of the following stages:
\begin{itemize}
\item  Augmented new incoming slice with few sampled indices of previous data to obtain $\tensor{X}_{new}^{`}$  .
\item \textbf{Sampling} : Using from $\tensor{X}_{new}^{`}$ extract sub-sample $\tensor{X}_{s}$   and common part $I_s,J_s$ and $K_s$ using \textit{Algorithm 1}.
\item \method : 
\begin{itemize}
\item Compute the \getrank method to find the rank of incoming slice $\tensor{X}_{new}^{`}$ using \textit{Algorithm 2}. 
\item  Run parallel decomposition on  $\tensor{X}_{s}$ using sampling factor \textit{s} using Algorithm 3. 
\item Update non zero entries of the non temporal mode $\mathbf{A_{old}}$  and $\mathbf{B_{old}}$ using $\mathbf{A_{new}}$  and $\mathbf{B_{new}}$ as resultant of \textit{Algorithm 1}. Update temporal mode $\mathbf{C}$ as :
 $\begin{bmatrix}
             \mathbf{C}_{old} \\
          \mathbf{C}_{new} \\
         \end{bmatrix}$
         \end{itemize}
\end{itemize}
}

\section{Experimental Evaluation}
\label{sec:experiments}

In this section we extensively evaluate the performance of \method on multiple synthetic and real datasets, and compare its performance with state-of-the-art approaches. We experiment on the different parameters of \method and the baselines, and how that affects performance. We implemented \method in Matlab using the functionality of the Tensor Toolbox for Matlab \cite{bader2015matlab}  which supports very efficient computations for sparse tensors. Our implementation is available at \codeurl. We used Intel(R) Xeon(R), CPU E5-2680 v3 @ 2.50GHz machine with 48 CPU cores and 378GB RAM. 
\subsection{Data-set description}

\subsubsection{Synthetic data generation}
\label{sec:syntDataDes}
In order to fully explore the performance of \method, in our experiments we generate synthetic tensor with varying density. We generate random tensors of dimension $I=J=K$ where I $\in$ [100, 500, 1000, 3000, 5000, 10000, 50000, 100000]. Those tensors are created from a known set of randomly generated factors, so that we have full control over the ground truth of the full decomposition. We dynamically calculate the size of batch or slice for our all experiments to fit the data into memory. For example for a $1000 \times 1000 \times 1000$ tensor, we selected batch size of 150 and for $5000\times 5000 \times$ 5000 tensor, we selected batch size of 100. The specifications of each synthetic dataset are given in Table \ref{table:tsyndataset}.
\begin{table}[h!]
\small
\begin{center}
\begin{tabular}{ |c|c|c|c|c| }
\hline
Dimension  & Density- & Density-& Batch & Sampling   \\
($I=J=K$)& dense & sparse& size &factor\\
\hline
\hline
 100  &100\%&65\%&50 &2\\
\hline
 500  &100\%&65\%&150&2\\
\hline
 1000  &100\%&55\%&150&2\\
\hline
 3000  &100\%&55\%&100&5\\
 \hline
 5000  &100\%&55\%&100&5\\
\hline
 10000 &100\%&55\%&10&2\\
  \hline
 50000  &100\%&35\%&5&2\\
\hline
 100000  &100\%&35\%&5&2\\
\hline
\end{tabular}
\bigskip
\caption{Table of Datasets analyzed}
\label{table:tsyndataset}
\end{center}
\vspace{-0.1in}
\end{table}
 
\subsubsection{Real Data Description}
\label{sec:realDataDec}
In order to truly evaluate the effectiveness of \method, we test its performance against six real datasets that have been used in the literature. Those datasets  are summarized in Table \ref{table:tdataset} and are publicly available at \url{http://frostt.io/tensors} \cite{frosttdataset}.

\hide{
NIPS, NELL, Facebook-Wall, Facebook-links, Patents, and Amazon dataset. These datasets are publicly available at \url{http://frostt.io/tensors}. \noindent{\textbf{NIPS}} publication dataset is collected by Globerson et al.\cite{chechik2007eec} and consists of papers published from 1987 to 2003 in NIPS. The modes are ``paper'', ``author'' and ``word relation''.
\noindent{\textbf{NELL}} \cite{carlson2010toward} is collected as part of the Read the Web project at Carnegie Mellon University. It is an entity-relation-entity tuple snapshot of the Never Ending Language Learner knowledge base. 
\noindent{\textbf{Facebook Wall posts}} dataset was first used in \cite{viswanath2009evolution} and has modes ``Wall owner'', ``Poster'', and
``day'', where the Poster created a post on the Wall owner’s Wall on the specified timestamp. 
\noindent{\textbf{Facebook Links posts}} \cite{viswanath2009evolution} contains a list of all of the user-to-user links from the Facebook New Orleans sub-network. 
\noindent{\textbf{Patents}} dataset is pairwise co-occurence of terms within window of 7 words in the US utility patents on a year basis. The modes of the tensor represents year-term-term, and the values are $\log(1+f_{i,j})$ , where $f_{i,j}$ is the frequency in which the words $i$ and $j$ appeared in the window.  Each slice of the tensor is symmetric.
Finally, \noindent{\textbf{Amazon}} dataset is collected by SNAP \cite{mcauley2013} and consists of product review from Amazon where mode are in form of user-product-word.
}

\begin{table*}[h!]
\small
\begin{center}
\begin{tabular}{ |c||c|c|c|c|c|c|c| }
\hline
Name&	Description &	Dimensions&	NNZ &  Batch & Sampling & Dataset   \\
&	 &	&	 &  size & factor & File Size  \\
\hline
\hline
NIPS \cite{chechik2007eec}&(Paper,Author,Word)	&2,482 x 2862 x 14036	&3,101,609&500&10&57MB\\

NELL	 \cite{carlson2010toward}&(Entity,Relation,Entity)	 &12092 x 9184 x 28818&76,879,419&500&10&1.4GB\\

Facebook-wall \cite{viswanath2009evolution}&(Wall owner, Poster, day)	&62,891 x 62,891 x 1,070&78,067,090&100&5&2.1GB\\

Facebook-links \cite{viswanath2009evolution}& (User, Links, Day)	&62,891	x  62,891 x	650	&263,544,295&50&2&3.8GB\\

Patents	& (Term ,Term, Year) & 239,172 x 239,172 x 46 &3,596,640,708&10&2&73GB\\

Amazon\cite{mcauley2013}	& (User, Product ,Word) &4,821,207 x	1,774,269 x 1,805,187&	1,741,809,018&50000&20&43GB\\
\hline
\end{tabular}
\bigskip
\caption{Real datasets analyzed}
\label{table:tdataset}
\end{center}
\vspace{-0.1in}
\end{table*}
\subsection{Evaluation Measures}
\label{sec:EvaMeas}
In order to obtain an accurate picture of the performance, we evaluate \method and the baselines using three criteria: Relative Error, Wall-Clock time and Fitness. These measures provide a quantitative way to compare the performance of our method.
More Specifically, \textbf{Relative Error} is effectiveness measurement and defined as : 
\[
Relative Error=\frac{||\tensor{X}_{original}-\tensor{X}_{predicted}||}{||\tensor{X}_{original}||}
\]
where, the lower the value, the better the approximation.

\noindent{\textbf{CPU time (sec)}} indicates how much faster does the decomposition runs as compared to re-running the entire decomposition algorithm whenever we receive a new update on the existing tensor. The average running time denoted by $T_{tot}$ for processing all slices for given tensor, measured in seconds, and is used to validate the time efficiency of an algorithm.

\noindent{\textbf{Relative Fitness}}: tracking the decomposition instead of recomputing it is inevitably an approximate task, so we would like to minimize the discrepancy of the incremental algorithm’s result from the one that has to re-compute the decomposition for every update. Relative Fitness is defined as: 
\[
Relative Fitness=\frac{||\tensor{X}_{original}-\tensor{X}_{\method}||}{||\tensor{X}_{original}-\tensor{X}_{BaseLine}||}
\]
where, the lower the value, the better the approximation for \method.
\subsection{Baselines for Comparison}
\label{sec:BasComControl}
Here we briefly present the state-of-the-art baselines we used for comparison. Note that for each baseline we use the {\em reported parameters} that yielded the best performance in the respective publications. For fairness, we compare against the parameter configuration for \method that yielded the best performance in terms of low wall-clock timing, low relative error and fitness. However, moving one step further, we evaluate \getrank and demonstrate that it qualitatively improve the performance for \method. We also show that using \getrank in \method we can still achieve better average running time when datasets are very large.  Note that all comparisons were carried out over 10 iterations each, and each number reported is an average with a standard deviation attached to it. \\
\noindent{\bf CP\_ALS \cite{bader2015matlab}}: is considered the most standard and well optimized implementation of the Alternating Least Squares algorithm for CP. We use the implementation of the Tensor Toolbox for Matlab \cite{bader2015matlab}. Here, we simply re-compute CP using CP\_ALS every time a new batch update arrives.

\noindent{\bf SDT \cite{nion2009adaptive}}: Simultaneous Diagonalization Tracking (SDT is based on incrementally tracking the Singular Value Decomposition (SVD) of the unfolded tensor $\tensor{X}_{(3)}=U\sum V^T$ and obtain $C=UW^{-1}$ and $\mathbf{D}=\mathbf{V}\sum W^T$. Matrix $\mathbf{A}$ (resultant left singular vector) and $\mathbf{B}$ (resultant right singular vector) are then estimated by applying SVD on each column $\hat{e_i}$ of $\mathbf{D}$. \\
\noindent{\bf RLST \cite{nion2009adaptive}}:Recursive Least Squares Tracking (RLST) is another online approach in which recursive updates are computed to minimize the Mean Squared Error (MSE) on incoming  slice. In RLST, $\mathbf{C}_{new}$ is computed as $\tensor{X}_{new}D_{old}^T$ and $\mathbf{C}$ is updated as  $ \left(\begin{smallmatrix} \mathbf{C}_{old}\\  \mathbf{C}_{new}\end{smallmatrix}\right)$. Then $\mathbf{D}$ is estimated using matrix inversion on $\tensor{X}_{new}$ and $C_{new}$. \\
\noindent{\bf OnilneCP \cite{zhou2016accelerating}}: The most recent andrelated work to ours was proposed by Zhou, \textit{el at.} \cite{zhou2016accelerating} is an OnlineCP decomposition method, where the the latent factors are updated when there are new data. OnilneCP fixes $\mathbf{A}$ and $\mathbf{B}$ to solve for $\mathbf{C}$ and minimizes the cost as:
\[\mathbf{C}\leftarrow \argmin_c \frac{1}{2} ||\left(\begin{smallmatrix} \tensor{X}_{old(3)}\\  \tensor{X}_{new(3)}\end{smallmatrix}\right)-\left(\begin{smallmatrix} \mathbf{C}_{old}\\  \mathbf{C}_{new}\end{smallmatrix}\right)(\mathbf{B} \odot \mathbf{A})^T||^2
\] 
\hide{
Then $\mathbf{C}$ is updated as :\\
\[
\mathbf{C}=\left(\begin{smallmatrix} \mathbf{C}_{old}\\  \mathbf{C}_{new}\end{smallmatrix}\right)=\left[\left(\begin{smallmatrix} \mathbf{C}_{old}\\  \tensor{X}_{new(3)}\end{smallmatrix}\right)-((\mathbf{B} \odot \mathbf{A})^T)^\ddagger\right]
\]
Matrix $\mathbf{A}$ is updated as $\mathbf{A}=\mathbf{P}\mathbf{Q}^{-1}$ where $\mathbf{P}=\mathbf{P}_{old}+\tensor{X}_{new(1)}(\mathbf{C}_{new} \odot \mathbf{B})$ and Q is $(\mathbf{C}^T\mathbf{C} \otimes \mathbf{B}^T\mathbf{B})$. Matrix $\mathbf{B}$ is estimated as $\mathbf{B}=\mathbf{U}\mathbf{V}^{-1}$ where $\mathbf{U}=\tensor{X}_{(2)}(\mathbf{C}\odot \mathbf{A})$ and $\mathbf{V}=(\mathbf{C} \odot \mathbf{A})^T (\mathbf{C} \odot \mathbf{A})$.\\
}


We conduct our experiments on multiple synthetic datasets and six real-world tensors datasets. We set the tolerance rate for convergence between consecutive iterations to $10^{-5}$ and the maximum number of iteration to 1000 for all the algorithms.The batch size and sampling factor is selected based on dimensions of first mode i.e. I, provided in table \ref{table:tsyndataset} and \ref{table:tdataset} for synthetic and real dataset respectively. 
We use the publicly available implementations for the baselines, as provided by the authors. We only modified the interface of the baselines, so that it is consistent across all methods with respect to the way that they receive the incoming slices. No other functionality has been changed.
\subsection{Experimental Results}
\label{sec:Experiments}
In this section, we experimentally evaluate the performance of the proposed method \method. The following major four aspects are analyzed.\\
\textbf{Q1:} How effective is \method as compared to the baselines on different synthetic and real world datasets.\\
\textbf{Q2:} How fast is \method when compared to the state-of-the-art methods on very large sized datasets?\\
\textbf{Q3:} What is the cost-benefit trade-off of computing the actual rank of the incoming batch?\\
\textbf{Q4:} What is the influence of sampling factor $s$ and number of sampling repetitions $r$ on \method?\\
\subsubsection{Baselines for Comparison}
\label{sec:baslineComp}
For all datasets we compute Relative Error,CPU time (sec) and Fitness. For \method, CP\_{ALS} , OnlineCP, RSLT and SDT we use 10\% of the data in each dataset as existing dataset. We experimented for both dense as well as sparse tensor to check the performance of our method. The results for the dense and sparse synthetic data are shown in Table \ref{table:denseRE} - \ref{table:sparseRE}. For each of datasets , the best result is shown in bold. OnlineCP, SDT and RLST address the issue very well. Compared
with CP\_{ALS}, SDT and RLST reduce the mean running time by up to 2x times and OnlineCP reduce mean time by up to 3x times for small dataset (I up to 3000). Performance of RLST was better than SDT algorithm on 8 out of 8 third-order synthetic tensor datasets. In fact, the efficiency (in terms of CPU time (sec)) of SDT is quite close to RLST. However, the main issue of SDT and RLST is their estimation of relative error and fitness.  For some datasets, such as $I=100$ and $I=3000$, they perform well, while for some others, they exhibit  poor fitness and relative error, achieving only nearly half of the fitness of other methods. For {\em small size} datasets , OnlineCP's efficiency and accuracy is better than all methods. As the dimension grows, however, the performance of OnlineCP method reduces,and particularly for datasets of dimension larger than  $5000 \times 5000 \times 5000$. Same behavior is observed for sparse tensors. \method is comparable to baselines for small dataset and outperformed the  baselines for large dataset. CP\_{ALS} is the only baseline able to execute dataset up to size $3000 \times 3000 \times 3000$. These results answer {\bf Q1} as the \method have comparable accuracy to other baseline methods.

\begin{table}[h!]
\ssmall
\begin{center}
\captionsetup[Table]{font=tiny,labelfont=tiny}
\begin{tabular}{ |c||c|c|c|c|c|}
\hline
I=J=K&$CP_{ALS}$ &OnlineCP&SDT&RLST&\method  \\ 
\hline
100&0.109 $\pm$ 0.01&\textbf{0.107$\pm$ 0.02}&	0.173$\pm$ 0.02&	0.151$\pm$ 0.02&0.115$\pm$ 0.02	\\
500&	\textbf{0.102 $\pm$ 0.09}&	\textbf{0.102$\pm$ 0.09}&	0.217$\pm$ 0.06	&0.217$\pm$ 0.06&\textbf{0.102$\pm$ 0.09}\\
1000&0.103$\pm$ 0.01&	0.103$\pm$ 0.01&	0.287$\pm$ 0.01	&0.296$\pm$ 0.01&\textbf{0.102$\pm$ 0.01}\\
3000&0.119$\pm$ 0.01&	\textbf{0.108$\pm$ 0.01}&0.189$\pm$ 0.01	&0.206$\pm$ 0.01&0.109$\pm$ 0.01\\
5000	&N/A&	0.122$\pm$ 0.002&0.201$\pm$ 0.002	&0.196$\pm$ 0.04&\textbf{0.115$\pm$ 0.009}\\
10000&N/A&	0.173$\pm$ 0.04&0.225$\pm$ 0.04	&0.252$\pm$ 0.06&\textbf{0.162$\pm$ 0.01}	\\
50000	&N/A&0.215$\pm$ 0.03&0.229$\pm$ 0.03&0.26$\pm$ 0.01&\textbf{0.169$\pm$ 0.01}\\
100000&N/A&N/A&	N/A&	N/A&\textbf{0.275 $\pm$ 0.00}\\
\hline
\end{tabular}
\bigskip
\caption{\bf{Experimental results for relative error for synthetic dense tensor. We see that \method gives comparable accuracy to baseline.}}
\label{table:denseRE}
\end{center}
\vspace{-0.3in}
\end{table}

\begin{figure}[!ht]
	\begin{center}
		\includegraphics[clip, trim=3.5cm 8cm 4cm 8cm, width  = 0.24\textwidth]{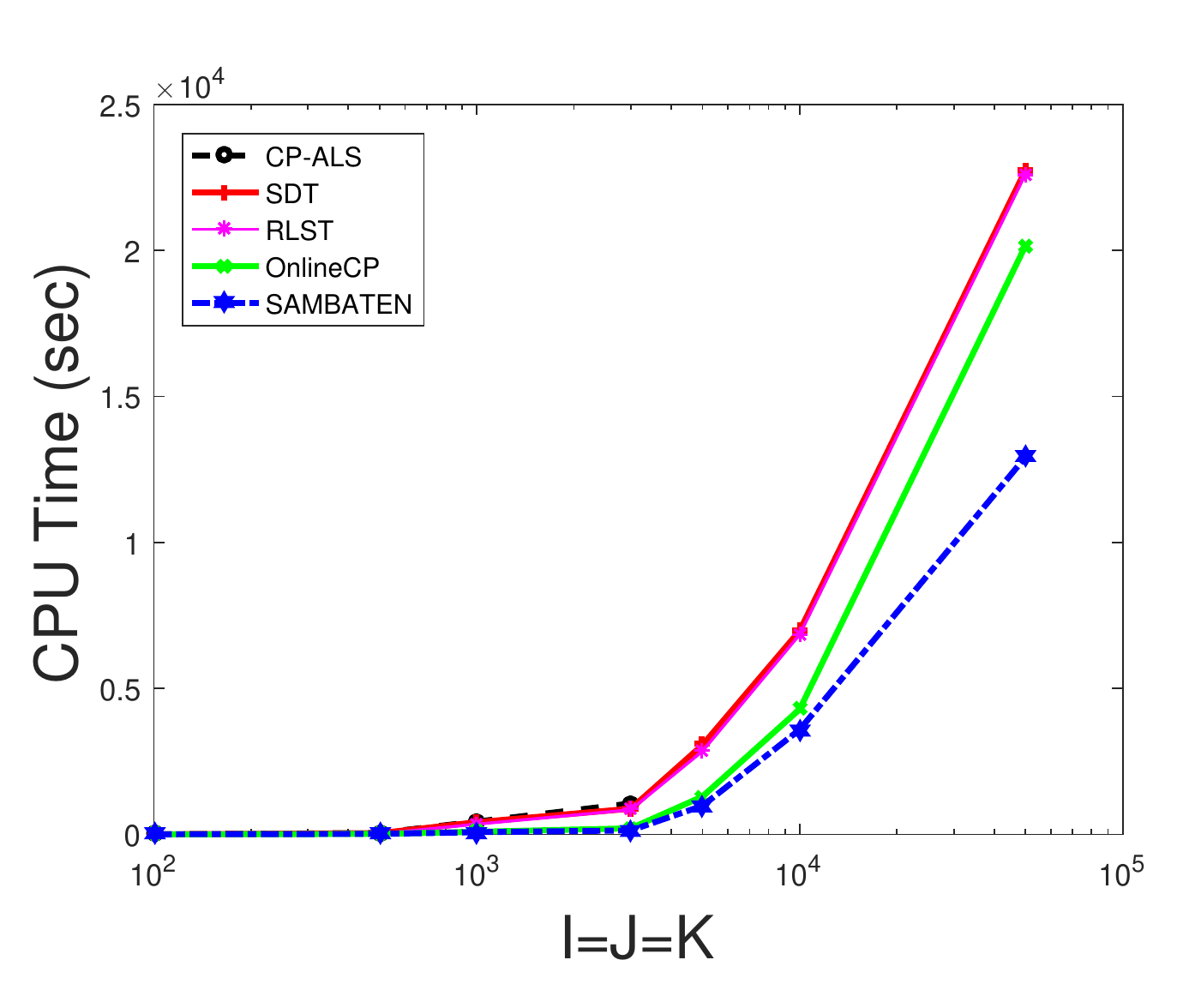}
		\includegraphics[clip, trim=3.5cm 8cm 4cm 8cm, width  = 0.24\textwidth]{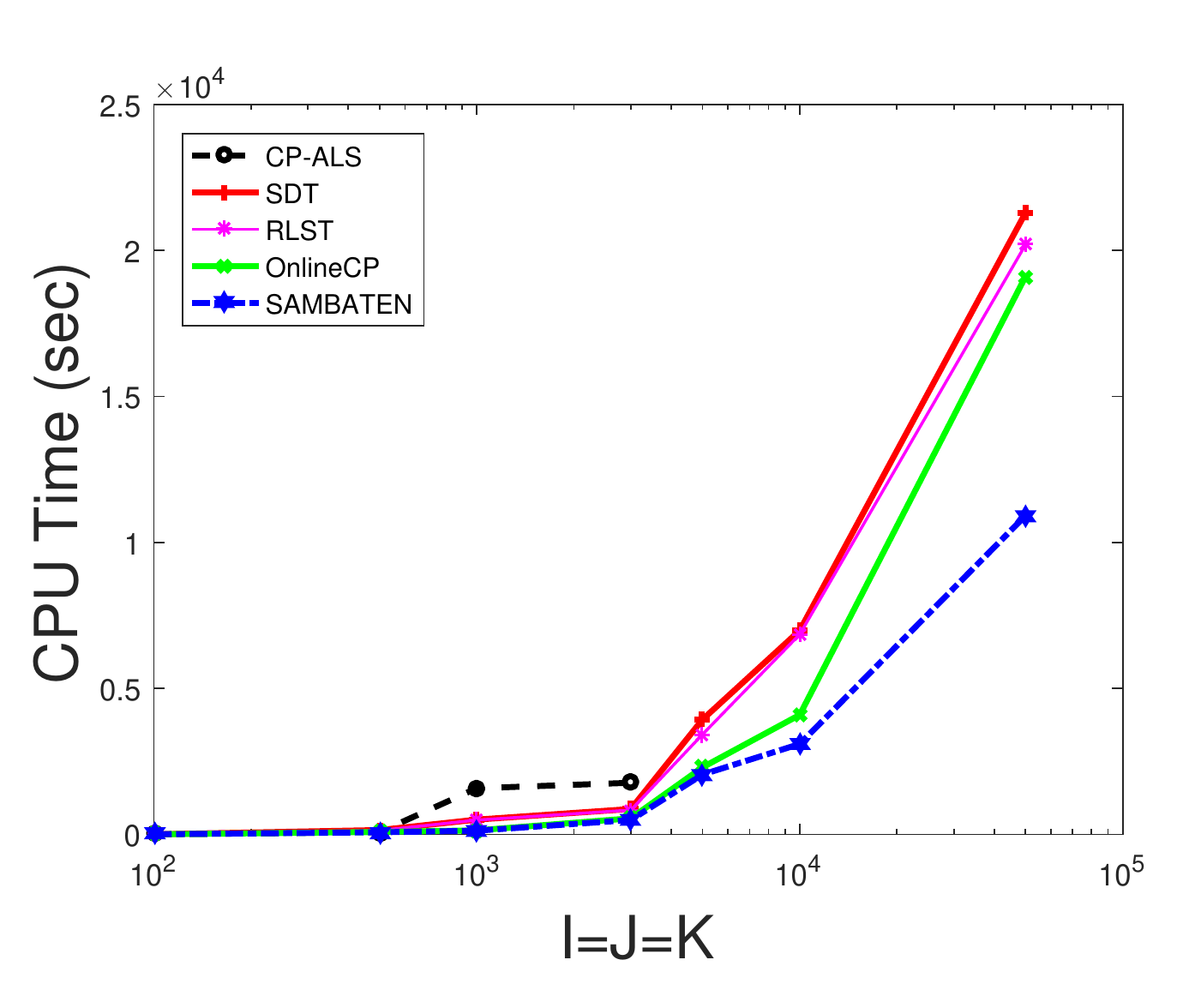}
		\caption{ \bf{Experimental results for CPU time (sec) for (a) dense tensor (b) sparse tensor}}
		\label{fig:denseCPU}
	\end{center}
	\vspace{-0.1in}
\end{figure}

\begin{figure}[!ht]
	\begin{center}
		\includegraphics[clip, trim=0.2cm 0.8cm 0.4cm 0.4cm,  width  = 0.24\textwidth]{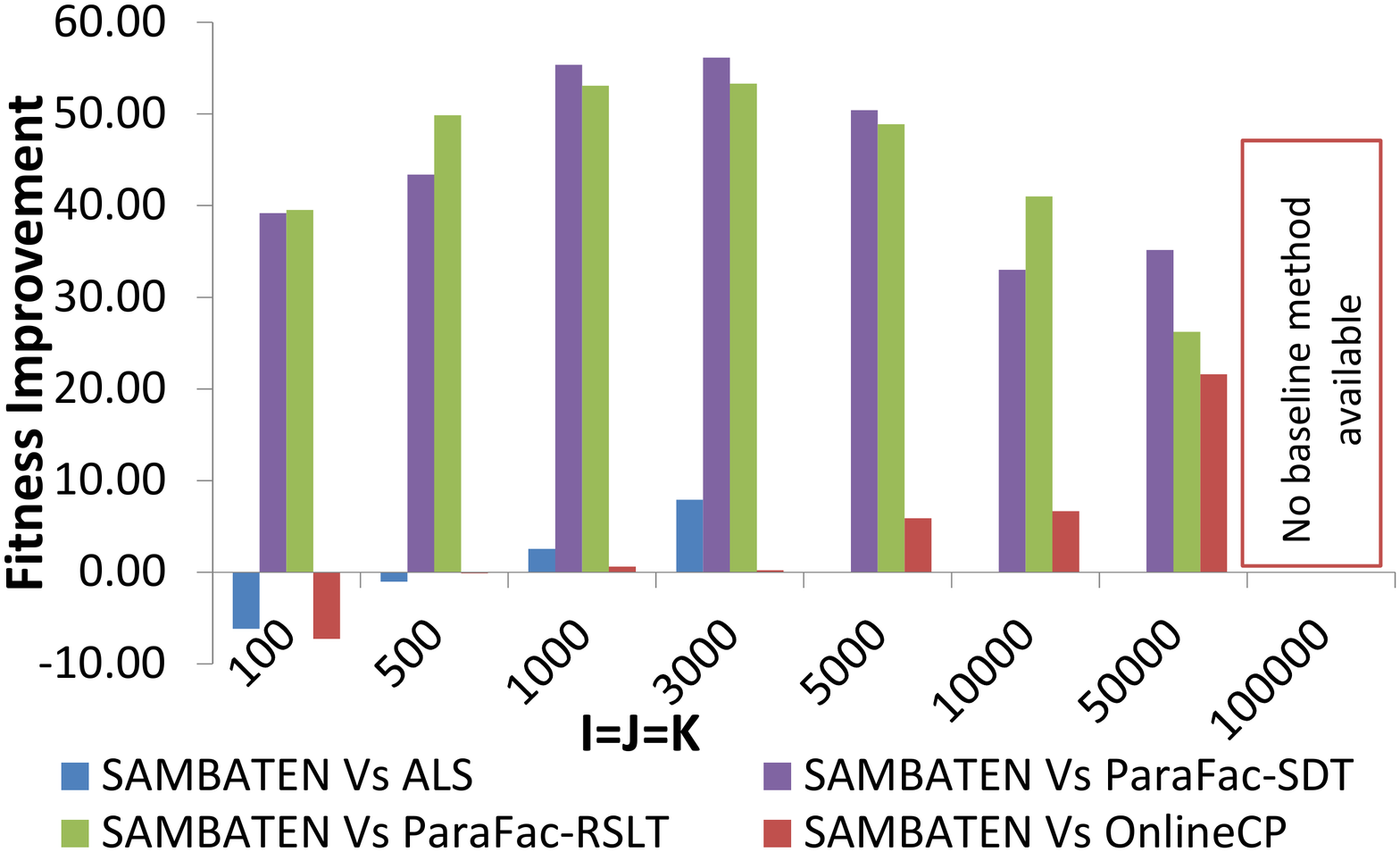}
		\includegraphics[ clip, trim=0.3cm 0.8cm 0.2cm 0.2cm, width  = 0.24\textwidth]{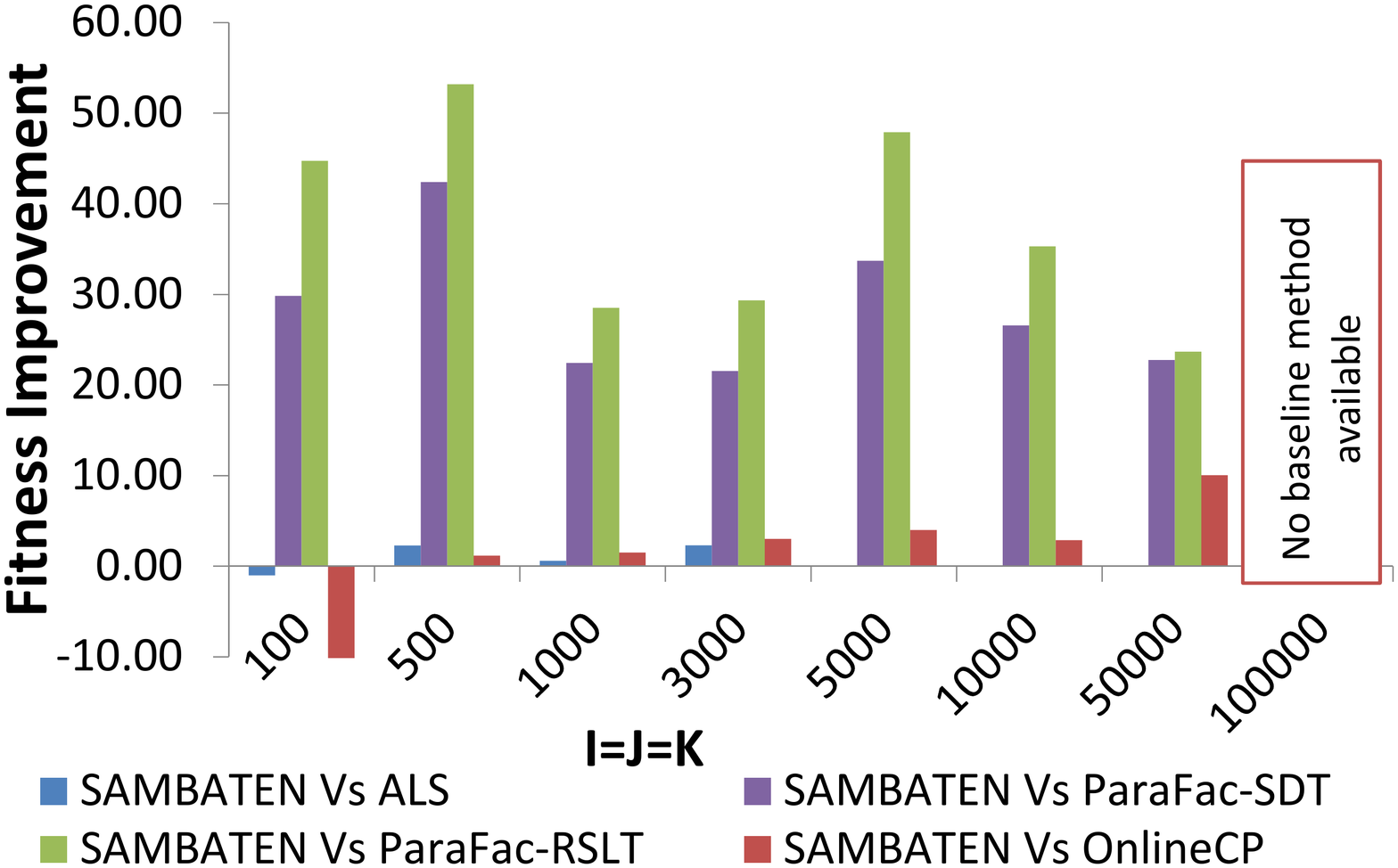}
		\caption{\bf{Experimental results for Relative Fitness Improvement for (a) dense tensor (b) sparse tensor}}
		\label{fig:spdenFitness}
	\end{center}
	\vspace{-0.1in}
\end{figure}

\begin{table}[ht!]
\ssmall
\captionsetup[Table]{font=tiny,labelfont=tiny}
\begin{center}
\begin{tabular}{ |c||c|c|c|c|c|}

\hline
I=J=K&$CP_{ALS}$ &OnlineCP&SDT&RSLT&\method  \\ 
\hline
100&0.169$\pm$ 0.01&\textbf{0.154$\pm$ 0.02}&0.306$\pm$ 0.01&0.313$\pm$ 0.01&0.178$\pm$ 0.01\\
500&\textbf{0.175$\pm$ 0.01}&0.188$\pm$ 0.01&0.43$\pm$ 0.01&0.421$\pm$ 0.01&0.184$\pm$ 0.01\\
1000&0.179$\pm$ 0.01&0.185$\pm$ 0.01&0.613$\pm$ 0.01&0.813$\pm$ 0.01&\textbf{0.178$\pm$ 0.01}\\
3000&0.177$\pm$ 0.02&\textbf{0.171$\pm$ 0.04}&0.446$\pm$ 0.03&0.513$\pm$ 0.02&0.176$\pm$ 0.03\\
5000&N/A&0.192$\pm$ 0.02&0.494$\pm$ 0.09&0.535$\pm$ 0.13&\textbf{0.187$\pm$ 0.04}\\
10000&N/A&\textbf{0.173$\pm$ 0.01}&0.212$\pm$ 0.01&0.224$\pm$ 0.11&0.198$\pm$ 0.12\\
50000&N/A&0.222$\pm$ 0.00&0.262$\pm$ 0.00&0.259$\pm$ 0.00&\textbf{0.200$\pm$ 0.00}\\
100000&N/A&N/A&	N/A&	N/A&\textbf{0.283$\pm$ 0.00}\\
\hline
\end{tabular}
\bigskip
\caption{\bf{Experimental results for relative error for synthetic sparse tensor. We see that \method works better in very large scale dataset such as 50000 $\times$ 50000 $\times$ 50000.}}
\label{table:sparseRE}
\end{center}
\vspace{-0.1in}
\end{table}
\method is efficiently able to compute $100K \times 100K \times 100K$ sized tensor with batch size of 5 and sampling factor 2. It took 58095.72s and 47232.2s to compute online decomposition for dense and sparse tensor, respectively. On other hand, state-of-art methods are unable to handle such large scaled incoming data.

The most interesting comparison, however, is on the real datasets, since they present more challenging cases than the synthetic ones. Table \ref{table:realCPU} shows the comparison between methods. \method outperforms other state-of-the-art approaches in most of  the real dataset. In the case of NIPS datset, \method gives better results compared to the baselines, specifically in terms of CPU Time ({\em faster up to 20 times}) and Fitness ({\em better up to 15-20\%}). \method outperforms for NELL, Facebook-Wall and Facebook-Links dataset in terms of efficiency comparable to CP\_{ALS}. For the NIPS dataset, \method is {\em 25-30 times faster} than OnlineCP method. Due to high dimensions of dataset, RSLT and SDT are unable to execute further. Note that all the real datasets we use are highly sparse, however, no baselines except CP\_ALS actually take advantage of that sparsity, therefore, repeated CP\_ALS tends to be faster because the baselines have to deal with dense computations which tend to be slower, when the data contain a lot of zeros.  Most importantly, however, \method performed very well on Amazon and Patent dataset, arguably the hardest of the six real datasets we examined and have been analyzed in the literature, {\em where none of the baselines was able to run}. These result answer {\bf Q1} and {\bf Q2} and show that \method is able to handle large dimensions and sparsity.
\begin{table*}[ht!]
\small
\begin{center}
\begin{tabular}{ |c||c|c|c|c|c||c|c|c|c|}
\hline
Dataset&\multicolumn{5}{|c||}{CPU Time (sec)}&\multicolumn{4}{|c|}{Fitness \method w.r.t} \\[0.5ex]
\hline
&$CP_{ALS}$ &OnlineCP&SDT&RSLT&\method  & $CP_{ALS}$&OnlineCP&SDT&RSLT \\ [0.5ex]
\hline
\hline
NIPS&177.46&372.03&1608.23&1596.07&\textbf{43.98}&0.96&0.98&\textbf{0.78}&0.82\\[0.5ex]
NELL&8783.27&42325.22&65325.22&63485.98&\textbf{983.83}&0.95&0.81&\textbf{0.76}&0.81\\[0.5ex]
Facebook-wall&3041.98&N/A&N/A&N/A&\textbf{736.07}&0.97&N/A&N/A&N/A\\[0.5ex]
Facebook-links&2689.69&N/A&N/A&N/A&\textbf{343.32}&0.96&N/A&N/A&N/A\\[0.5ex]
Amazon&N/A&N/A&N/A&N/A&\textbf{4892.07}&N/A&N/A&N/A&N/A\\[0.5ex]
Patent  &N/A&N/A&N/A&N/A&\textbf{8068.27} &N/A&N/A&N/A&N/A\\[0.5ex]
\hline
\end{tabular}
\bigskip
\caption{\bf{\method performance for real dataset. We see that \method outperformed the baselines for all the large scaled tensors.}}
\label{table:realCPU}
\end{center}
\vspace{-0.1in}
\end{table*}

\subsubsection{Evaluation of Quality Control}
\label{sec:qualityControl}
\begin{table}
\small
\begin{center}
\begin{tabular}{ |c||c|c|c|c|c|}
\hline
I=J=K&200&400&600&800&1000\\[1ex]
\hline
\hline
 w/  \getrank&0.48&0.57&0.58&0.59&0.55\\[1ex]
w/o \getrank&0.46&0.53&0.55&0.54&0.52\\[1ex]
\hline
\end{tabular}
\bigskip
\caption{\bf{FMS score for synthetic dataset of batch size 50 with sampling factor 2 for each dimension.}}
\label{table:SelRsyn}
\end{center}

\end{table}

\begin{table}
\small
\begin{center}
\begin{tabular}{ |c||c|c|c|c|c|c|}
\hline
Dataset & Sampling Factor&2&5&10&15&20\\ 
\hline
\hline
\multirow{2}{*}{NIPS} & w/  \getrank &0.26&0.53&0.45&0.48&0.36\\[1ex]
&w/o \getrank&0.24&0.46&0.36&0.24&0.22\\[1ex]
\multirow{2}{*}{NELL} & w/  \getrank&0.48&0.37&0.48&0.43&0.26\\[1ex]
&w/o \getrank&0.25&0.16&0.38&0.37&0.24\\[1ex]
\hline
\end{tabular}
\bigskip
\caption{\bf{FMS score for NIPS and NELL dataset with batch size 500, $R=5$, and same sampling factor for each dimension.}}
\label{table:SelRreal}
\end{center}
\vspace{-0.2in}
\end{table}

Here we evaluate the performance improvement brought by \getrank, as well as investigate the additional computational overhead associated with it.
 We perform experiments as shown in Figure \ref{fig:getRank} on different dataset to examine the cost in terms of CPU time (sec) which we pay to compute new rank of incoming slice. To measure the accuracy of \getrank , we compute the Fitness Improvement and the Factor Matching Score (FMS) score. We define FMS score as follows. If $\mathbf{A}$ and $\mathbf{B}$ are single component tensors that have been normalized so that their weights are $\lambda_a$ and $\lambda_b$, then the score is defined as score = penalty * ($a_1^{T}*b_1$) * $a_2^{T}*b_2$) * ... * ($a_R^{T}*b_R$) where the penalty is defined by the $\lambda$ values such that penalty =  $1 -\frac{|\lambda_a - \lambda_b|}{max(\lambda_a , \lambda_b)}$. FMS score is measured between 0 to 1, with 1 being ``perfect'' match. More precisely , FMS score is defined as :
\begin{equation}
\small
FMS Score=100*\sum_{r=1}^R \left(1 -\frac{|\lambda_a - \lambda_b|}{max(\lambda_a , \lambda_b)}\right) \prod_{n=1}^N|a_r^{(n)T}b_r^{(n)}|
\end{equation}

\begin{figure}[!ht]
	\begin{center}
		\includegraphics[ trim=3.5cm 8cm 4cm 8cm,width  = 0.23\textwidth]{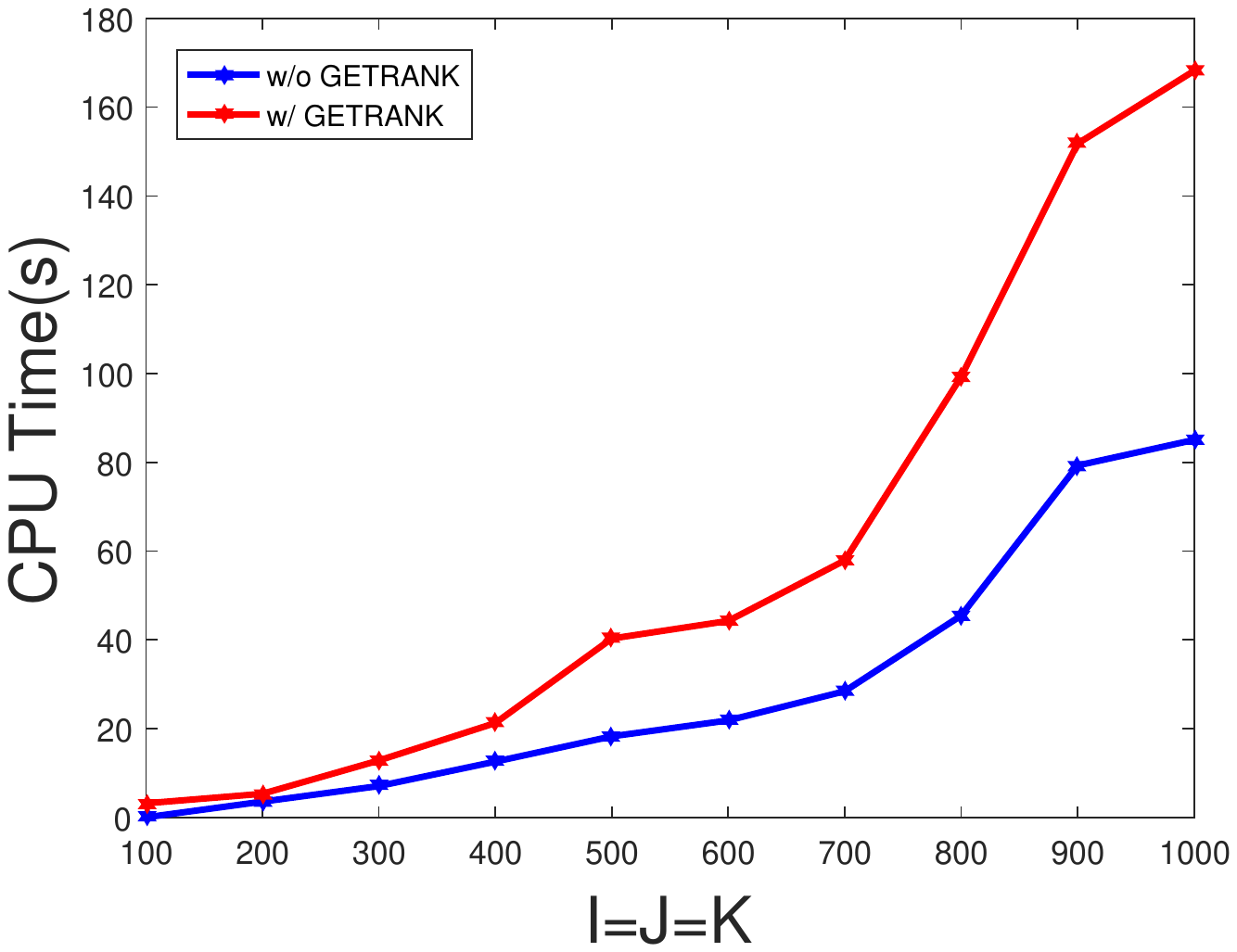}
		\includegraphics[ trim=3.5cm 8cm 4cm 8cm,width  = 0.23\textwidth]{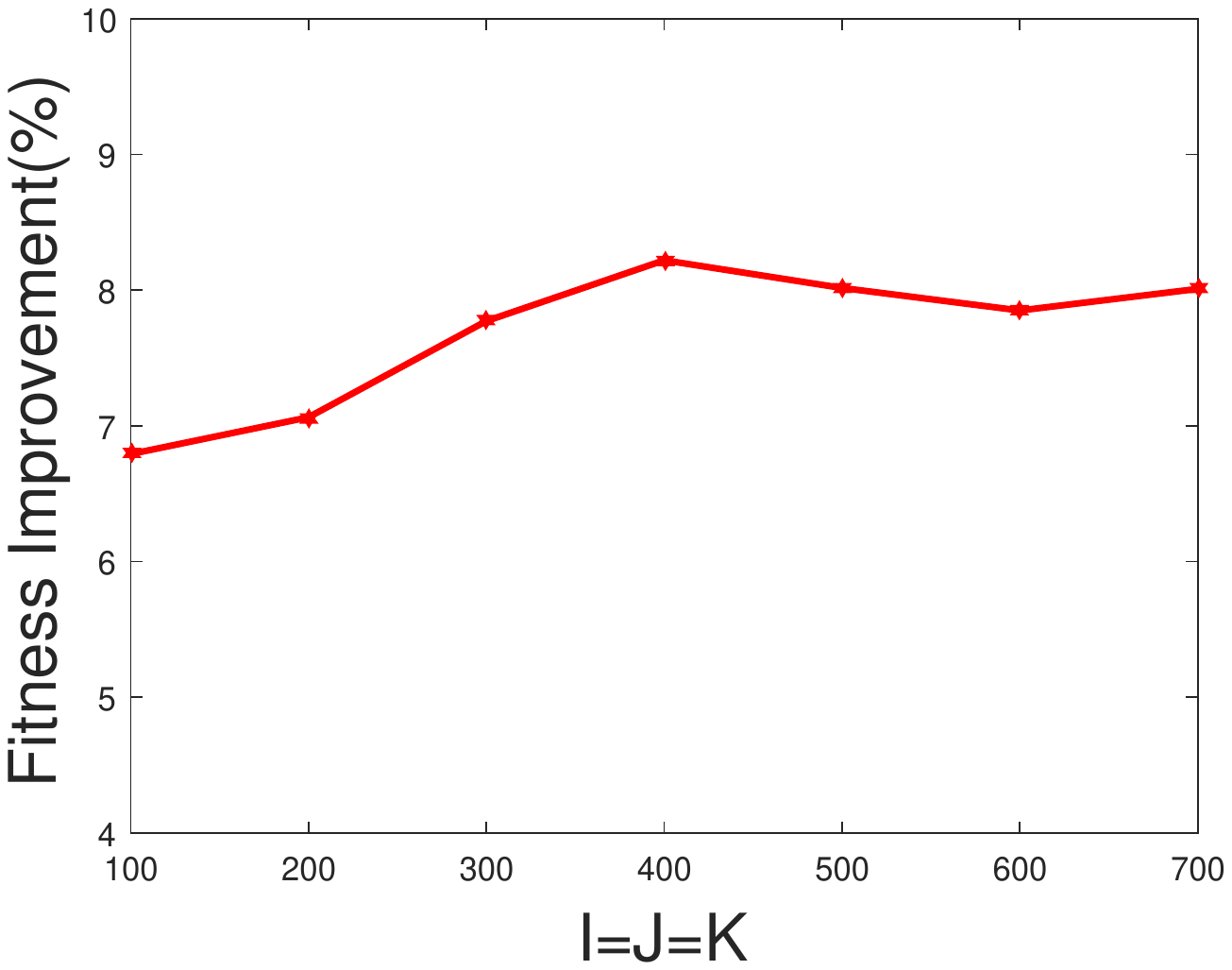}
						\caption{\bf{Experimental results for CPU Time (sec) and Relative Fitness Improvement using \getrank for synthetic dataset. Sampling factor  \textbf{s} is 2 and batch size is 50 for each synthetic dataset.}}
		\label{fig:getRank}
	\end{center}
	\vspace{-0.1in}
\end{figure}

\begin{figure}[!ht]
	\begin{center}
		\includegraphics[ trim=3.5cm 8cm 4cm 8cm,width  = 0.23\textwidth]{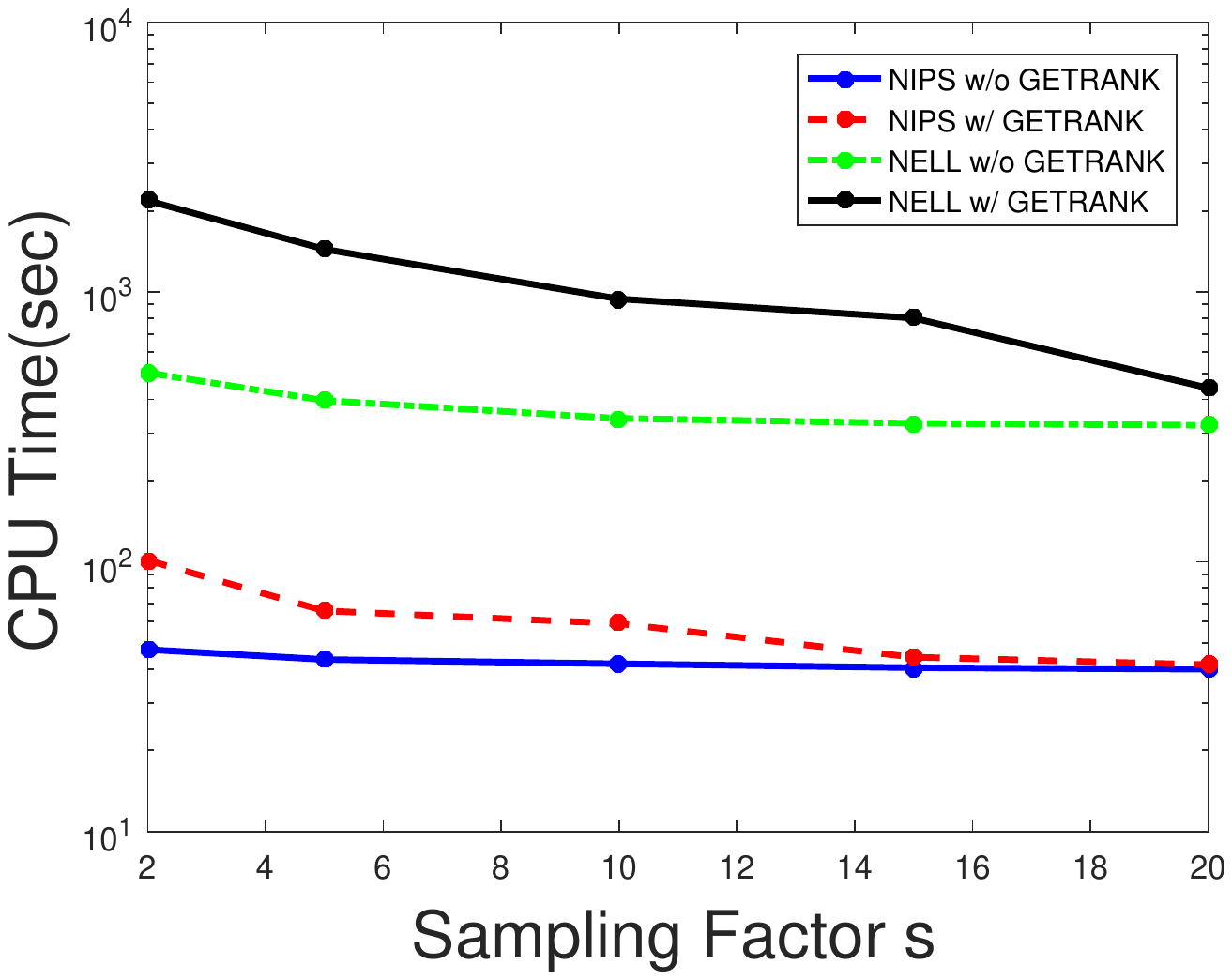}
		\includegraphics[ trim=3.5cm 8cm 4cm 8cm,width  = 0.23\textwidth]{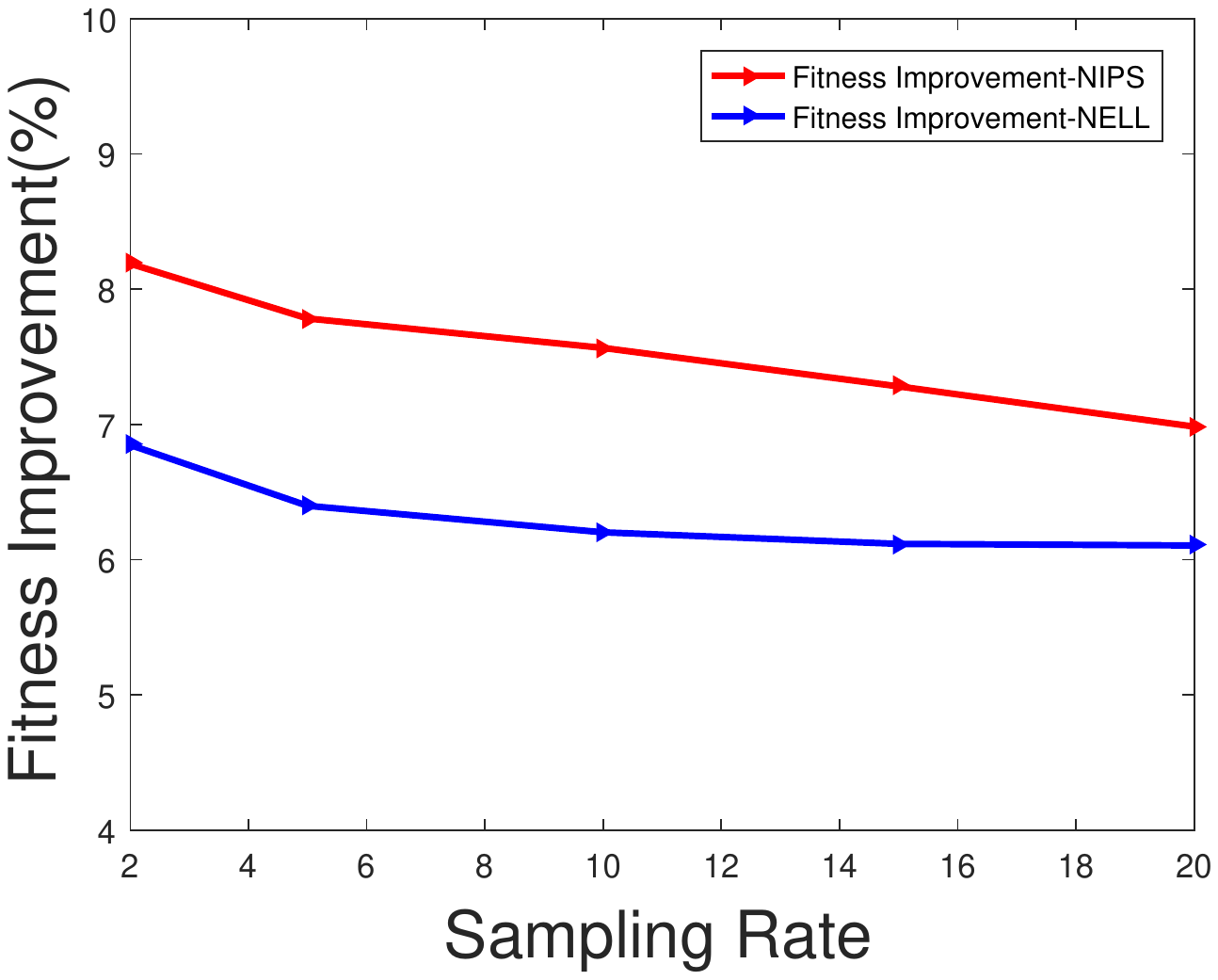}
				\caption{\bf{Experimental results for CPU Time (sec) and Relative Fitness Improvement using \getrank for NIPS and NELL dataset. Sampling factor  \textbf{s} $\in$ [2, 5, 10, 15, 20] with fixed batch size of 500.}}
		\label{fig:getRankReal}
	\end{center}
	\vspace{-0.1in}
\end{figure}

We compare the performance of \method without and with \getrank on synthetic data (where we know the actual components) and on real data, where we compute CP\_ALS on the full tensor and we set those as ground truth components. We observe that results consistently better in terms of FMS score \ref {table:SelRsyn}-\ref {table:SelRreal} and Fitness Improvement Figure \ref{fig:getRank} for synthetics and  Figure \ref{fig:getRankReal} for real NIPS dataset when we incorporate \getrank in our \method, without sacrificing run-time significantly. This answers {\bf Q3}.
\subsubsection{Tuning of Sampling Factor \textit{s}}
\label{sec:sControl}
The sampling factor plays an important role in \method. We performed experiments to evaluate the impact of changing sampling factor on \method. For these experiments , we fixed batch size to 50 for all datasets. We see in figure \ref{fig:samplingImpact} that increasing sampling factor results in reduction of CPU time (as sparsity of sub sampled tensor increased) and it reduces the fitness of output up-to 2-3\%. In sum, these observations demonstrate that: 1) a suitable sampling factor on sub-sampled tensor could improve the fitness and result in better tensor decomposition, and 2) the higher sampling factor is, the lower the CPU time. This result partially answers {\bf Q4}.

\begin{figure}[!ht]
	\begin{center}
		\includegraphics[ trim=3.5cm 8cm 4cm 8cm,width  = 0.23\textwidth]{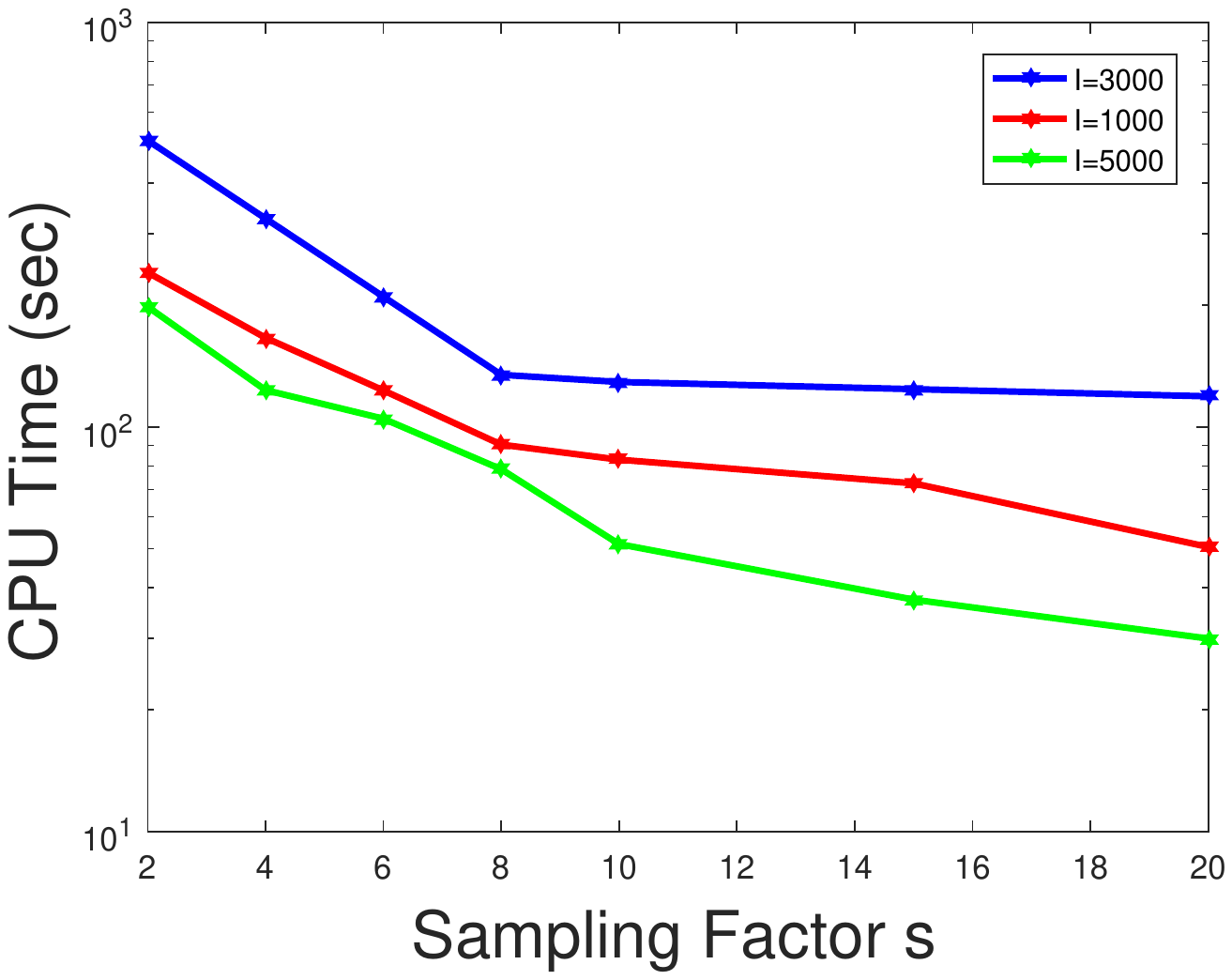}
		\includegraphics[ trim=3.5cm 8cm 4cm 8cm,width  = 0.23\textwidth]{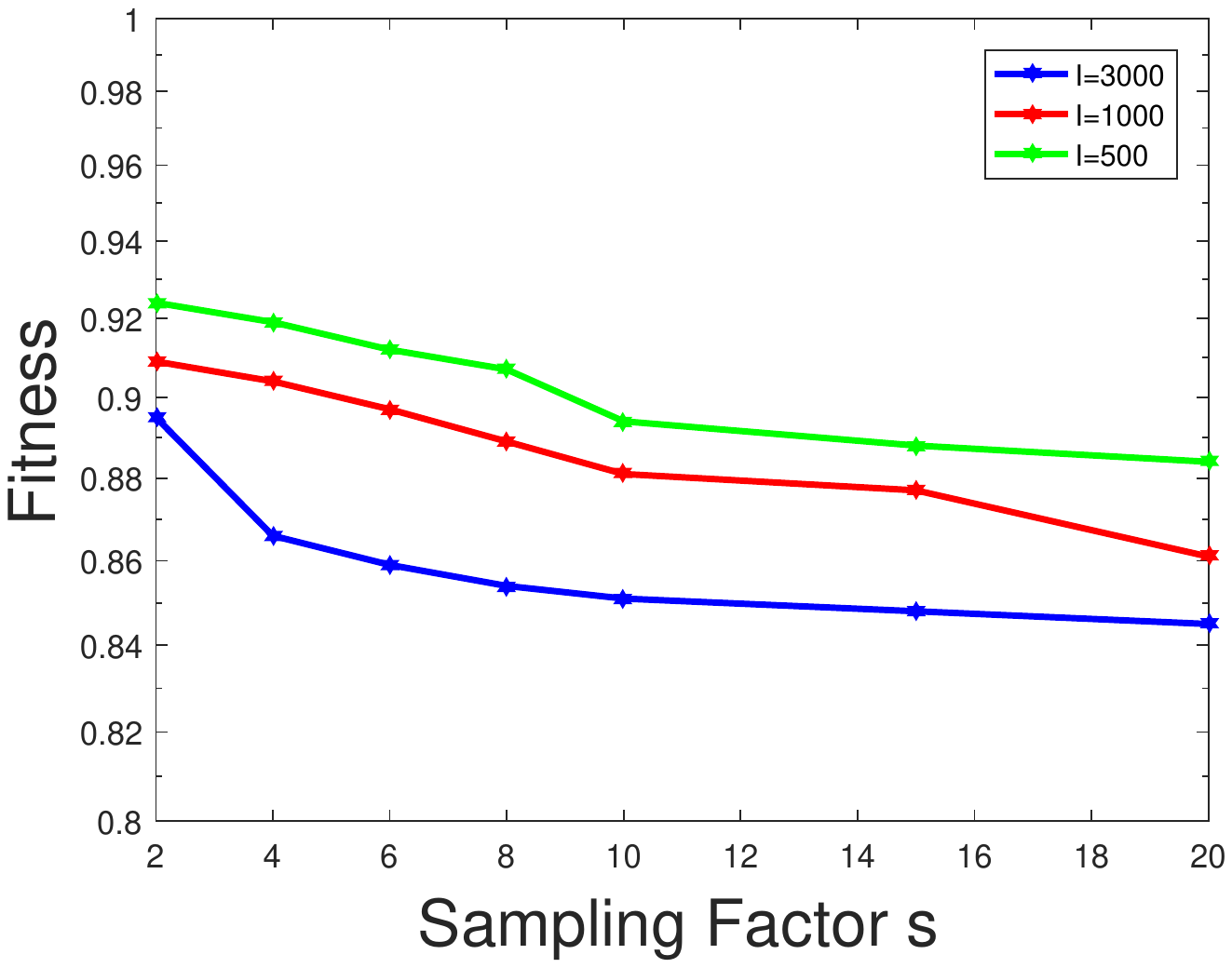}
		\caption{\bf{\method outputs sampling factor : CPU Time (sec)  and Relative Fitness on different datasets.}}
		\label{fig:samplingImpact}
	\end{center}
	\vspace{-0.1in}
\end{figure}

\begin{figure}[!ht]
	\begin{center}
		\includegraphics[ trim=3.5cm 8cm 4cm 8cm,width  = 0.23\textwidth]{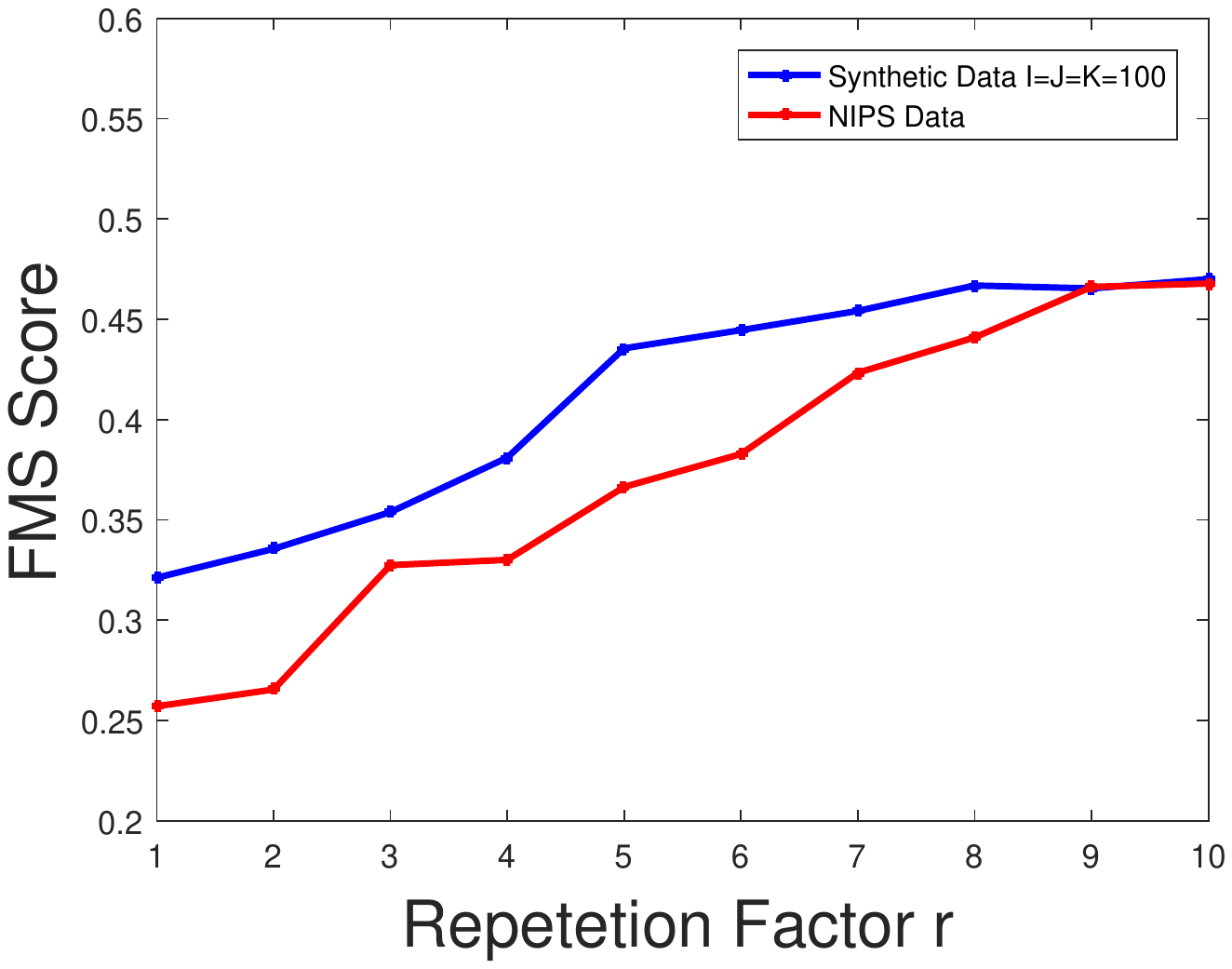}
		\includegraphics[ trim=3.5cm 8cm 4cm 8cm,width  = 0.23\textwidth]{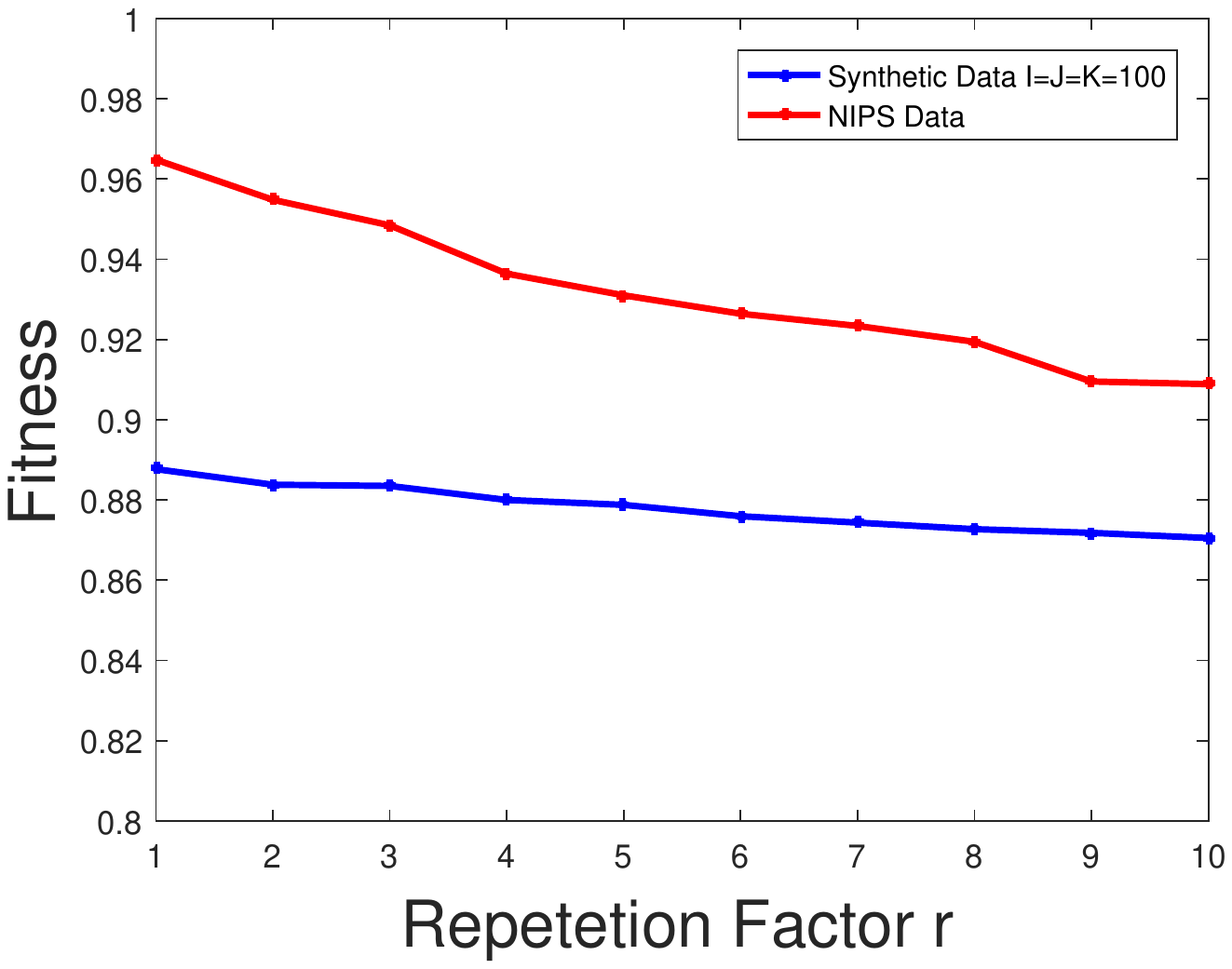}
		\caption{\bf{\method outputs repetition factor \textit{r}:  FMS score and Relative Fitness on synthetic and NIPS real world datasets.}}
		\label{fig:rImpact}
	\end{center}
	\vspace{-0.1in}
\end{figure}

\subsubsection{Influence of Repetition Factor \textit{r}}
\label{sec:rControl}
We evaluate the performance for parameter setting $r$ i.e number of paralleldecompositions.  For these experiments, we choose batch size and sampling rate for synthetic $ 500 \times 500 \times 500$ dataset and  NIPS real world dataset as provided in table \ref{table:tsyndataset} and \ref{table:tdataset}, respectively. We can see that with higher values of the repetition factor $r$, FMS score  and Relative Fitness (\method vs CP\_{ALS}) is  improved as shown in figure \ref{fig:rImpact}. We experiment on varying repetition factor \textit{r} with Sampling factor \textit{s} on NIPS real world dataset to check the performance of our method as shown in figure \ref{fig:rImpactSampling}. Note that higher FMS score indicates a better decomposition and, similarly, the lower the fitness score, the better decomposition. This result completes the answer to {\bf Q4}.

\begin{figure}[!ht]
	\begin{center}
		\includegraphics[ trim=3.5cm 8cm 4cm 8cm,width  = 0.23\textwidth]{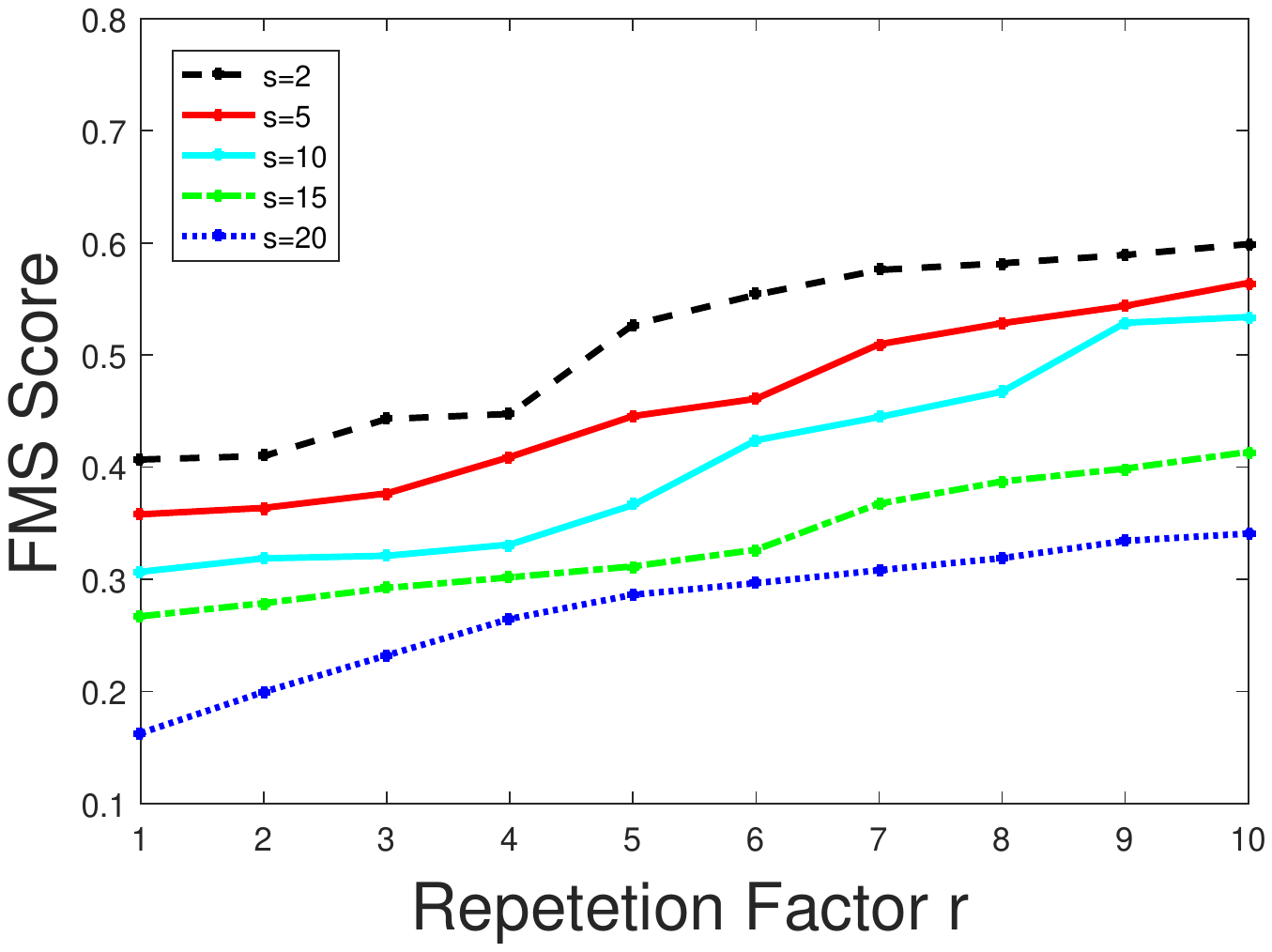}
		\includegraphics[ trim=3.5cm 8cm 4cm 8cm,width  = 0.23\textwidth]{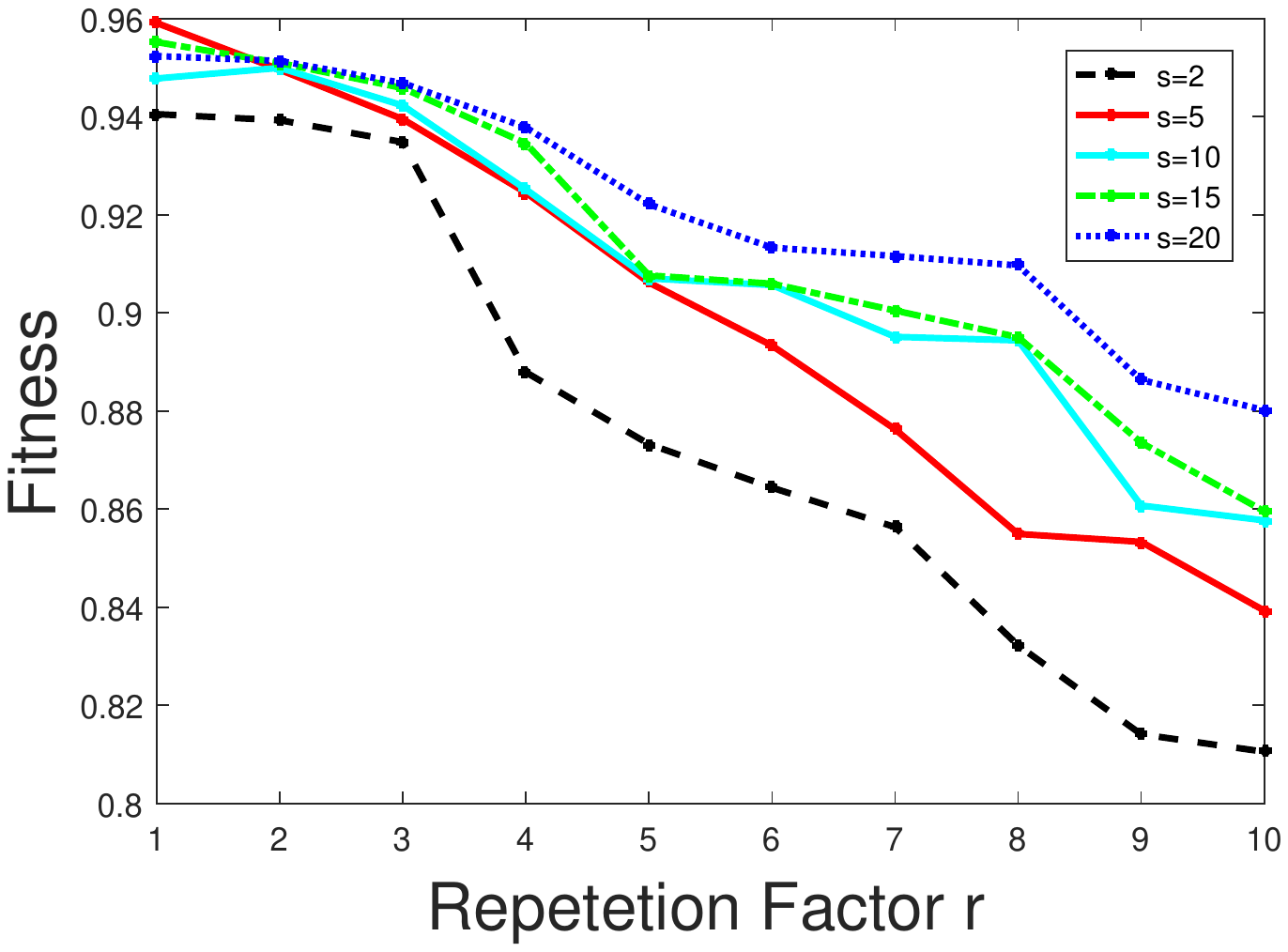}
		\caption{\bf{\method outputs for repetition factor \textit{r} and Sampling factor \textit{s}:  FMS score and Relative Fitness  for NIPS real world dataset.}}
		\label{fig:rImpactSampling}
	\end{center}
	\vspace{-0.1in}
\end{figure}
\section{Related Work}
\label{sec:related}
In this section, we provide review of the work related to our algorithm. At large,  incremental tensor methods in the literature can be categorized into three main categories:  1) Tucker
decomposition, 2) CP decomposition, 3) Tensor completion

\noindent{\bf Tucker Decomposition}:
Online tensor decomposition was first proposed by Sun \textit{el at.}\cite{SunITA} as ITA (Incremental Tensor Analysis). In there research, they described the three variants of Incremental Tensor Analysis. First,  DTA i.e. Dynamic tensor analysis which is based on calculation of co-variance of matrices in traditional higher-order singular value decomposition in an incremental fashion. Second, with help of SPIRIT algorithm, they found approximation of DTA named as Stream Tensor Analysis (STA). Third, they proposed window-based tensor analysis (WTA). To improve the efficiency of DTA, it uses a sliding window strategy. 

Liu \textit{el at.}\cite{papadimitriou2005streaming} proposed an efficient method to diagonalize the core tensor to overcome this problem. Other approaches replace SVD with incremental SVD to improve the efficiency. Hadi \textit{el at.} \cite{fanaee2015multi} proposed multi-aspect-streaming tensor analysis (MASTA) method that relaxes  constraint and allows the tensor to concurrently evolve through all modes.  

\noindent{\bf CP Decomposition}:
There is very limited study on online CP decomposition methods. Phan \textit{el at.} \cite{phan2011parafac} had developed a theoretic approach GridTF to large-scale tensors processing based on an extension to CP's fundamental mathematics theory. They used divide and concur technique to get sub-tensors and fuses the output of all factorization to achieve final factor matrices which is proved to be same as decomposing the whole tensor using CP decomposition. Its potential of concurrent computing methods to adapt the engineering applications remains unclear. Sidiropoulos \textit{el at.}\cite{nion2009adaptive}, proposed two algorithms that focus on CP decomposition namely SDT (Simultaneous Diagonalization Tracking) that incrementally perform the SVD of the unfolded tensor; and RLST (Recursive Least Squares Tracking), which recursively updates the decomposition factors by minimizing the mean squared error. The most related work to ours was proposed by Zhou, \textit{el at.} \cite{zhou2016accelerating} is an online CP decomposition method, where the the latent factors are updated when there are new data. 

\noindent{\bf Tensor Completion}:
Related to Tucker and CP decomposition are formulations which are focused on Tensor Completion, i.e., the estimation of missing values in a tensor. The main difference between completion and decomposition techniques is that in completion ``zero'' values are considered ``missing'' and are not part of the model, and furthermore, the goal is to impute those missing values accurately, rather than extracting latent factors or subspaces that can describe the existing (observed) data. To the best of our knowledge, the earliest work on incremental tensor completion traces back to  \cite{mardani2015subspace}, and very recently,  Qingquan \textit{el at.}\cite{song2017multi},  proposed a streaming tensor completion algorithm based on block partitioning of the tensor.

\hide{
\begin{figure}[!ht]
	\begin{center}
		\includegraphics[width = 0.22\textwidth]{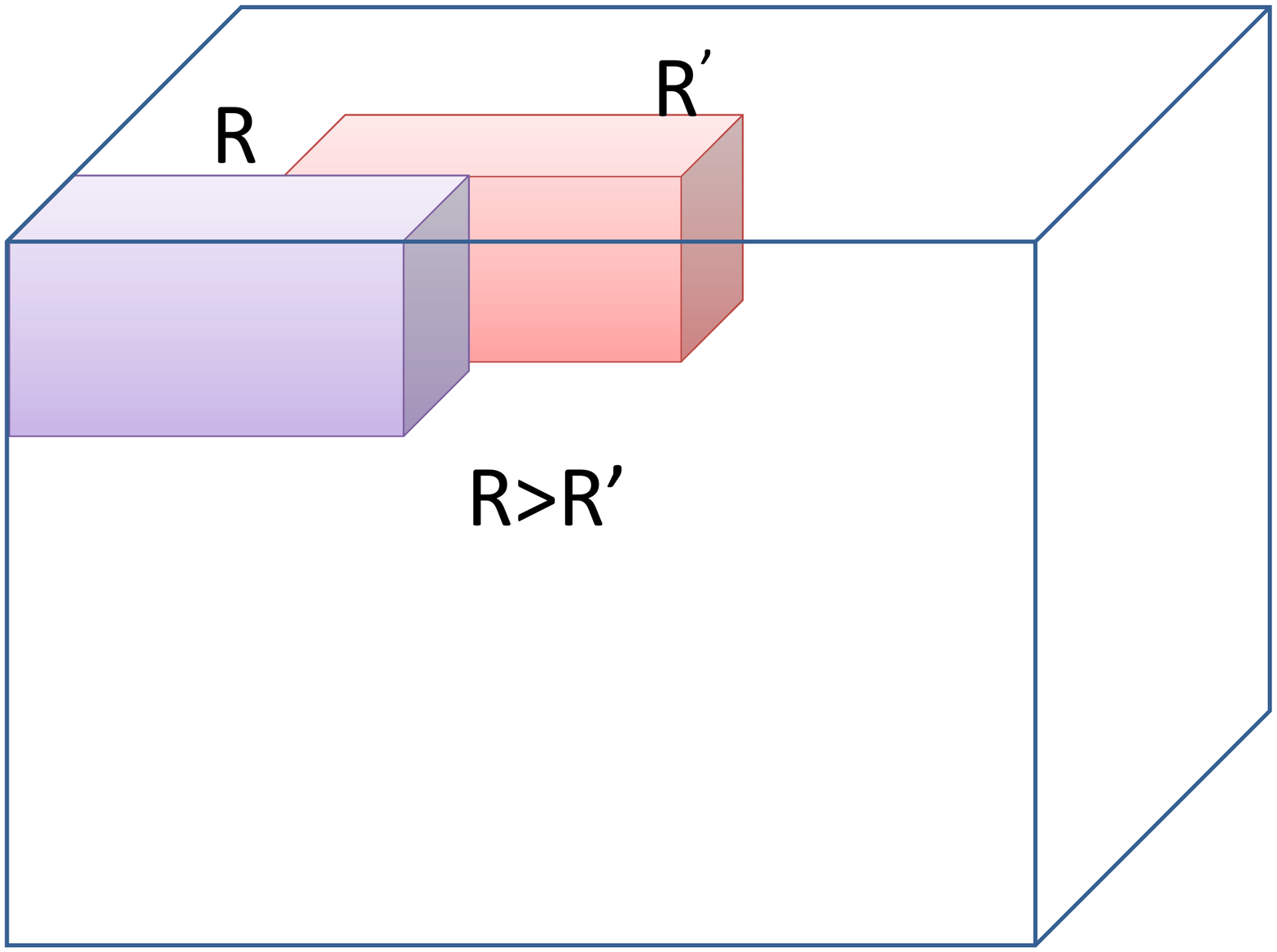}
	      \caption{diffRankSlice}
             \label{fig:diffRankSlice}
	\end{center}
\end{figure}
}
\section{Conclusions}
\label{sec:conclusions}
In this paper we introduce \method, a novel sample-based incremental CP tensor decomposition. We show its effectiveness with respect to approximation quality, with its performance being on par with state-of-the-art incremental and non-incremental algorithms, and we demonstrate its efficiency and scalability by outperforming state-of-the-art approaches ({\em 25-30 times faster}) and being able to run very large incremental tensors where none of the baselines was able to produce results. In the future we intend to explore different tensor decompositions that can also benefit from our proposed algorithmic framework.

\hide{
\section{Acknowledgments}
{\scriptsize
Research was supported by the National Science Foundation Grant No. XXXXXX. Any opinions, findings, and conclusions or recommendations expressed in this material are those of the author(s) and do not necessarily reflect the views of the funding parties.
}
}

\balance
\bibliographystyle{plain}
\bibliography{BIB/vagelis_refs}

\begin{thebibliography}{10}

\bibitem{pmlrv51anandkumar16}
Anima Anandkumar, Prateek Jain, Yang Shi, and U.~N. Niranjan.
\newblock Tensor vs. matrix methods: Robust tensor decomposition under block
  sparse perturbations.
\newblock In Arthur Gretton and Christian~C. Robert, editors, {\em Proceedings
  of the 19th International Conference on Artificial Intelligence and
  Statistics}, volume~51 of {\em Proceedings of Machine Learning Research},
  pages 268--276, Cadiz, Spain, 09--11 May 2016. PMLR.

\bibitem{bader2015matlab}
Brett~W Bader, Tamara~G Kolda, et~al.
\newblock Matlab tensor toolbox version 2.6, available online, february 2015,
  2015.

\bibitem{bro1997parafac}
R.~Bro.
\newblock Parafac. tutorial and applications.
\newblock {\em Chemometrics and intelligent laboratory systems},
  38(2):149--171, 1997.

\bibitem{bro2003new}
Rasmus Bro and Henk~AL Kiers.
\newblock A new efficient method for determining the number of components in
  parafac models.
\newblock {\em Journal of chemometrics}, 17(5):274--286, 2003.

\bibitem{carlson2010toward}
Andrew Carlson, Justin Betteridge, Bryan Kisiel, Burr Settles, Estevam~R.
  Hruschka~Jr., and Tom~M. Mitchell.
\newblock Toward an architecture for never-ending language learning.
\newblock In {\em AAAI}, volume~5, page~3, 2010.

\bibitem{carroll1970analysis}
J~Douglas Carroll and Jih-Jie Chang.
\newblock Analysis of individual differences in multidimensional scaling via an
  n-way generalization of “eckart-young” decomposition.
\newblock {\em Psychometrika}, 35(3):283--319, 1970.

\bibitem{chi2012tensors}
Eric~C Chi and Tamara~G Kolda.
\newblock On tensors, sparsity, and nonnegative factorizations.
\newblock {\em SIAM Journal on Matrix Analysis and Applications},
  33(4):1272--1299, 2012.

\bibitem{erdos2013walk}
D{\'o}ra Erdos and Pauli Miettinen.
\newblock Walk'n'merge: A scalable algorithm for boolean tensor factorization.
\newblock In {\em Data Mining (ICDM), 2013 IEEE 13th International Conference
  on}, pages 1037--1042. IEEE, 2013.

\bibitem{papalexakis2012parcube}
{Evangelos E. Papalexakis}, Christos Faloutsos, and Nicholas~D Sidiropoulos.
\newblock Parcube: Sparse parallelizable tensor decompositions.
\newblock In {\em ECML-PKDD'12}.

\bibitem{fanaee2015multi}
Hadi Fanaee-T and Jo{\~a}o Gama.
\newblock Multi-aspect-streaming tensor analysis.
\newblock {\em Knowledge-Based Systems}, 89:332--345, 2015.

\bibitem{chechik2007eec}
A.~Globerson, G.~Chechik, F.~Pereira, and N.~Tishby.
\newblock {Euclidean Embedding of Co-occurrence Data}.
\newblock {\em The Journal of Machine Learning Research}, 8:2265--2295, 2007.

\bibitem{PARAFAC}
R.A. Harshman.
\newblock Foundations of the parafac procedure: Models and conditions for an"
  explanatory" multimodal factor analysis.
\newblock 1970.

\bibitem{kolda2005higher}
Tamara~G Kolda, Brett~W Bader, and Joseph~P Kenny.
\newblock Higher-order web link analysis using multilinear algebra.
\newblock In {\em Data Mining, Fifth IEEE International Conference on}, pages
  8--pp. IEEE, 2005.

\bibitem{kolda2009tensor}
T.G. Kolda and B.W. Bader.
\newblock Tensor decompositions and applications.
\newblock {\em SIAM review}, 51(3), 2009.

\bibitem{mardani2015subspace}
Morteza Mardani, Gonzalo Mateos, and Georgios~B Giannakis.
\newblock Subspace learning and imputation for streaming big data matrices and
  tensors.
\newblock {\em IEEE Transactions on Signal Processing}, 63(10):2663--2677,
  2015.

\bibitem{mcauley2013}
Julian McAuley and Jure Leskovec.
\newblock Hidden factors and hidden topics: understanding rating dimensions
  with review text.
\newblock In {\em Proceedings of the 7th ACM conference on Recommender
  systems}, pages 165--172. ACM, 2013.

\bibitem{nion2009adaptive}
D.~Nion and N.D. Sidiropoulos.
\newblock Adaptive algorithms to track the parafac decomposition of a
  third-order tensor.
\newblock {\em Signal Processing, IEEE Transactions on}, 57(6):2299--2310,
  2009.

\bibitem{papadimitriou2005streaming}
Spiros Papadimitriou, Jimeng Sun, and Christos Faloutsos.
\newblock Streaming pattern discovery in multiple time-series.
\newblock In {\em Proceedings of the 31st international conference on Very
  large data bases}, pages 697--708. VLDB Endowment, 2005.

\bibitem{papalexakis2015fast}
Evangelos~E Papalexakis and Christos Faloutsos.
\newblock Fast efficient and scalable core consistency diagnostic for the
  parafac decomposition for big sparse tensors.
\newblock In {\em Acoustics, Speech and Signal Processing (ICASSP), 2015 IEEE
  International Conference on}, pages 5441--5445. IEEE, 2015.

\bibitem{papalexakis2016tensors}
Evangelos~E Papalexakis, Christos Faloutsos, and Nicholas~D Sidiropoulos.
\newblock Tensors for data mining and data fusion: Models, applications, and
  scalable algorithms.
\newblock {\em ACM Transactions on Intelligent Systems and Technology (TIST)},
  8(2):16, 2016.

\bibitem{phan2011parafac}
Anh~Huy Phan and Andrzej Cichocki.
\newblock Parafac algorithms for large-scale problems.
\newblock {\em Neurocomputing}, 74(11):1970--1984, 2011.

\bibitem{sidiropoulos2004low}
Nikos~D Sidiropoulos.
\newblock Low-rank decomposition of multi-way arrays: A signal processing
  perspective.
\newblock In {\em Sensor Array and Multichannel Signal Processing Workshop
  Proceedings, 2004}, pages 52--58. IEEE, 2004.

\bibitem{frosttdataset}
Shaden Smith, Jee~W. Choi, Jiajia Li, Richard Vuduc, Jongsoo Park, Xing Liu,
  and George Karypis.
\newblock {FROSTT}: The formidable repository of open sparse tensors and tools,
  2017.

\bibitem{song2017multi}
Qingquan Song, Hancheng~Ge Xiao~Huang, James Caverlee, and Xia Hu.
\newblock Multi-aspect streaming tensor completion.
\newblock In {\em Proceedings of the 23th ACM SIGKDD international conference
  on Knowledge discovery and data mining}. ACM, 2017.

\bibitem{SunITA}
Jimeng Sun, Dacheng Tao, Spiros Papadimitriou, Philip~S. Yu, and Christos
  Faloutsos.
\newblock Incremental tensor analysis: Theory and applications.
\newblock {\em ACM Trans. Knowl. Discov. Data}, 2(3):11:1--11:37, October 2008.

\bibitem{ten2002uniqueness}
Jos~MF ten Berge and Nikolaos~D Sidiropoulos.
\newblock On uniqueness in candecomp/parafac.
\newblock {\em Psychometrika}, 67(3):399--409, 2002.

\bibitem{viswanath2009evolution}
Bimal Viswanath, Alan Mislove, Meeyoung Cha, and Krishna~P Gummadi.
\newblock On the evolution of user interaction in facebook.
\newblock In {\em Proceedings of the 2nd ACM workshop on Online social
  networks}, pages 37--42. ACM, 2009.

\bibitem{zhou2016accelerating}
Shuo Zhou, Nguyen~Xuan Vinh, James Bailey, Yunzhe Jia, and Ian Davidson.
\newblock Accelerating online cp decompositions for higher order tensors.
\newblock In {\em Proceedings of the 22nd ACM SIGKDD International Conference
  on Knowledge Discovery and Data Mining}, pages 1375--1384. ACM, 2016.

\end{thebibliography}

\end{document}